\PassOptionsToPackage{noend}{algorithm2e}
\documentclass[12pt]{colt2018_mod}

\usepackage[utf8]{inputenc} 
\usepackage[T1]{fontenc}    
\usepackage{amsfonts}       
\usepackage{nicefrac}       
\usepackage{microtype}      

\usepackage{color}
\usepackage{latexsym}
\usepackage{stmaryrd}    

\usepackage{algorithm}
\usepackage{subcaption}
\usepackage{thmtools, thm-restate}

\newtheorem*{theorem*}{Theorem}

\usepackage{etoolbox,xparse,xspace,setspace}
\usepackage{array,multirow,booktabs} 
\usepackage{bbm} 

\DeclareMathOperator{\E}{\mathbb{E}}
\DeclareMathOperator*{\opmax}{\vee}
\DeclareMathOperator*{\opmin}{\wedge}

\newcommand{\F}{\mathcal{F}}
\newcommand{\G}{\mathcal{G}}
\newcommand{\W}{\mathcal{W}}

\newcommand{\rad}{\mathcal{R}}
\newcommand{\good}{\textsc{good}}
\newcommand{\bad}{\textsc{bad}}
\newcommand{\loss}{\ell}
\newcommand{\reals}{\mathbb{R}}
\newcommand{\X}{\mathcal{X}}
\newcommand{\Y}{\mathcal{Y}}
\newcommand{\Z}{\mathcal{Z}}

\newcommand{\N}{\mathcal{N}}
\newcommand{\D}{\mathcal{D}}
\newcommand{\B}{\mathcal{B}}

\newcommand{\nth}{^{\text{th}}}
\newcommand{\ind}[1]{\left\llbracket #1 \right\rrbracket}
\newcommand{\Prob}{\mathsf{P}\mathop{}}
\newcommand{\Pn}{\mathsf{P}_n\mathop{}}

\newcommand{\cPn}{\tilde{\mathsf{P}}_n\mathop{}} 
\newcommand{\cf}{\tilde{f}} 

\renewcommand{\Y}{\mathcal{Y}}

\newcommand{\Var}{\mathop{\mathrm{Var}}}
\newcommand{\ns}{N}  

\newcommand{\argmin}{\operatornamewithlimits{argmin}}

\newcommand{\algto}{~\textup{\textbf{to}}~}
\newcommand{\ucb}{\mathrm{ucb}}
\newcommand{\tinyeq}{\!=\!}

\newcommand{\given}{\;|\;}
\renewcommand{\epsilon}{\varepsilon}

\makeatletter
\DeclareRobustCommand{\qed}{%
  \ifmmode 
  \else \leavevmode\unskip\penalty9999 \hbox{}\nobreak\hfill
  \fi
  \quad\hbox{\qedsymbol}}
\newcommand{\qedsymbol}{\BlackBox}
\newenvironment{Proof}[1][\proofname]{\par
  \normalfont
  \topsep6\p@\@plus6\p@ \trivlist
  \item[\hskip\labelsep\bfseries
    #1]\ignorespaces
}{%
  \qed\endtrivlist
}
\newcommand{\proofname}{Proof}
\makeatother

\title[Multi-Observation Regression]{Multi-Observation Regression}
\usepackage{times}

\coltauthor{\Name{Rafael Frongillo} \Email{raf@colorado.edu}\\
  \addr University of Colorado, Boulder
  \AND
  \Name{Nishant A. Mehta} \Email{nmehta@uvic.ca}\\
  \addr University of Victoria
  \AND
  \Name{Tom Morgan} \Email{tdmorgan@seas.harvard.edu}\\
  \addr Harvard University
  \AND
  \Name{Bo Waggoner} \Email{bwag@seas.upenn.edu}\\
  \addr University of Pennsylvania
}

\begin{document}

\maketitle

\begin{abstract}
  Recent work introduced loss functions which measure the error of a prediction based on multiple simultaneous observations or outcomes.
  In this paper, we explore the theoretical and practical questions that arise when using such multi-observation losses for regression on data sets of $(x,y)$ pairs. 
When a loss depends on only one observation, the average empirical loss decomposes by applying the loss to each pair, but for the multi-observation case, empirical loss is not even well-defined, and the possibility of statistical guarantees is unclear without several $(x,y)$ pairs with exactly the same $x$ value. 
  We propose four algorithms formalizing the concept of empirical risk minimization for this problem, two of which have statistical guarantees in settings allowing both slow and fast convergence rates, but which are out-performed empirically by the other two.
  Empirical results demonstrate practicality of these algorithms in low-dimensional settings, while lower bounds demonstrate intrinsic difficulty in higher dimensions.
  Finally, we demonstrate the potential benefit of the algorithms over natural baselines that use traditional single-observation losses via both lower bounds and simulations.
\end{abstract}


\section{Introduction}
\label{sec:intro}

The inference method of empirical risk minimization (ERM) selects a hypothesis by minimizing some loss function over a data set.
The question of which loss to use in which situation, however, is often a matter of debate.
Traditional loss functions have a familiar form: $\loss(f(x),y)$, where $f$ is a hypothesis, $x$ is a feature vector of some kind, and $y$ is a label or observation.
Notably, such losses score the accuracy of a hypothesis's prediction $f(x)$ based on a single observation $y$.

Recent work~\citep{casalaina-martin2017multi-observation} proposed a natural generalization of ERM which depends on \emph{multiple} observations: $\loss(f(x),y_1,y_2,\ldots,y_m)$.
These \emph{multi-observation} losses\footnote{Outside of the regression setting considered here, multi-observation losses have also appeared in the context of learning embeddings~\citep{hadsell2006dimensionality,schroff2015facenet:,ustinova2016learning}, and some techniques from time series analysis can be viewed as a special case of multi-observation regression~\citep{engle1986modelling}.}
  allow more natural and direct capture of ``higher-order'' properties of the conditional distribution on $y$ given $x$.
Their advantages can be formalized via \emph{elicitation complexity}~\citep{lambert2008eliciting,frongillo2015elicitation-2,casalaina-martin2017multi-observation}.
For example, to regress on the variance of $y$ given $x$, any statistically consistent loss function can be shown to require either a multi-dimensional range of $f(x)$ or multiple observations.
Further, some statistics, such as the 2-norm of the conditional distribution, can be modeled much more simply with multiple observations than with multiple dimensions~\citep{casalaina-martin2017multi-observation}.
In this paper, we give two more examples: the upper confidence bound (UCB), and MINVAR~\citep{cherny2009new}.
See \S~\ref{sec:background} for more on elicitation complexity.

Despite the groundwork laid by~\cite{casalaina-martin2017multi-observation}, a very important methodological question remains: given a data set of $(x,y)$ pairs, how should one use such a multi-observation loss function to actually perform ERM?
For the single-observation case, up to regularization, ERM is simply $\argmin_{f\in\F} \sum_{(x,y)\in\D} \loss(f(x),y)$.
But the analogue in the multi-observation case is $\argmin_{f\in\F} \sum \loss(f(x), y_1,\dots,y_m)$, requiring data of the form $(x,y_1,y_2,\ldots,y_m)$, i.e., multiple observations or labels $y$ for every feature vector $x$.
While data of this form is sometimes obtainable, we would like to perform ERM in the typical case where each $x$ has only one associated $y$ value.

\paragraph{Contributions.}
In this paper, we present several algorithms to perform multi-observation regression given data sets of $(x,y)$ pairs. These algorithms reduce to ERM after constructing \emph{metasamples} of the form $(x,y_1,\ldots,y_m)$.
The first two are \emph{unbiased} algorithms, which take care to ensure that the $x$ values are drawn independently from the underlying distribution.
We also present two \emph{biased} algorithms which perform better in practice but whose statistical guarantees are much harder to analyze due to the lack of i.i.d.\ $x$ values.

For our unbiased algorithms, we prove statistical convergence guarantees, with both slow and fast rates depending on the setting. 
Thus, we give the first excess risk bounds for ERM with multi-observation losses, aside from very preliminary (and much weaker) bounds of~\cite{casalaina-martin2017multi-observation}.
The key technique is an analysis of ERM with \emph{corrupted samples}, i.e.~samples from distributions near but not equal to the same underlying distribution.
We consider both labeled and unlabeled sample complexity and discuss the importance of the distinction under the multi-observation paradigm. 
Our algorithmic and sample complexity contributions are most valuable and practically relevant in the low-dimensional regime; as we show in \S~\ref{sec:dimension} and \S~\ref{sec:stone-lower-bound}, poor dependence on the dimension $d$ of $x$ is information-theoretically unavoidable when predicting higher-order statistics.

Finally, we demonstrate the advantage of multi-observation regression over the traditional single-observation approach in cases such as predicting conditional variance.
The intuition is that the single-observation approach requires (in some sense) fitting several surrogate statistics (e.g.~first and second moment), which may follow some very complex relationships that are learning-theoretically difficult to fit.
We give both theoretical and empirical (simulation) examples. 
We often use the variance as an example due to its simplicity, but we stress that our results are fully general, encompassing the UCB, MINVAR, and countless other properties elicitable via multi-observation losses.

\paragraph{Applications.}
Aside from augmenting the literature on fundamental properties of ERM, our results have applications to robust engineering design, uncertainty quantification, and finance.
In the design of an airfoil, building, truss, etc., engineers seek designs which minimize some objective (drag, cost, weight, displacement), but which are robust to changes in the environment or to manufacturing defects~\citep{beyer2007robust,royset2016set-based}.
To understand the quality of a design, expensive computer simulations are performed, which often have a stochastic component, and the goal is to minimize some \emph{risk measure} of the objective, such as the 95\% quantile of the drag for an airfoil~\citep{jouhaud2007surrogate-model}.
Our algorithms could increase the statistical efficiency of the popular practice of \emph{surrogate optimization}, where a hypothesis $f(x)$ is fit to the risk measure, for design parameters $x$, and then $f$ is minimized directly~\citep{shahbaz2016surrogate-based}.
We show that two popular risk measures for robust design, the upper confidence bound~\citep{doltsinis2004robust} and MINVAR~\citep{ong2006max-min} are easily fit with multi-observation losses.

In finance, risk measures are used both in decision-making and to regulate banks and other financial institutions~\citep{artzner1999coherent,follmer2004stochastic}.
Typically, banks must report estimates of their risk to authorities alongside their 10-day monetary losses~\citep{campbell2005review}, and a statistical test is performed to judge the accuracy of these risk estimates~\citep{acerbi2014back-testing}.
To perform this test, however, one essentially needs a loss function which is statistically consistent with the risk measure, i.e., which ``elicits'' it~\citep{emmer2013what,gneiting2011making,fissler2015expected}; this has played a major role in academic responses to the recent Basel 3.5 risk regulation proposal~\citep{carver2014back-testing,embrechts2014academic}.
Previous work~\citep{casalaina-martin2017multi-observation} has shown that multi-observation loss functions broaden the class of elicitable risk measures, but of course a bank cannot report ``i.i.d.''\ monetary losses for the same day.
Our results could imply that using multiple monetary losses from nearby days would suffice.

\section{Background: ERM and Elicitation}
\label{sec:background}

Classically, in supervised learning an algorithm is presented with an i.i.d.~sample of $n$ labeled points $(X_1, Y_1), \ldots, (X_n, Y_n)$ with the objective of selecting some action $f \in \F$ that obtains low expected loss $\E_{(X, Y) \sim \D \times \D_X} \left[ \loss_f(X, Y) \right]$, or risk, with respect to some loss function $\loss$. 
Here, we use the notation $\loss_f(X, Y) = \loss(f(X), Y)$, so that each $f$ is a function mapping from $\X$ to $\reals$, and we denote by $f^* \in \F$ the risk minimizer over $\F$. 
We also always assume that the loss $\loss$ is $L$-Lipschitz in its first argument. 
ERM, which returns any $\hat{f} \in \F$ that minimizes the empirical risk $\frac{1}{n} \sum_{i=1}^n \loss_f(X, Y)$, is a natural choice for solving this problem. The performance of ERM is now known to be tightly characterized by the notion of Rademacher complexity. 

\begin{definition}[Rademacher complexity]
Let $\G$ be a class of functions mapping from a space $\Z$ to $\reals$, and let $Z_1, \ldots, Z_n$ be an i.i.d.~sample from distribution $P$ over $\Z$. 
Let $\epsilon_1, \ldots, \epsilon_n$ be independent Rademacher random variables (distributed uniformly on $\{-1, 1\}$). The Rademacher complexity of $\G$ (with respect to $P$) is
\begin{align*}
  \rad_n(\G) = \E \left[ \E_{\epsilon_1, \ldots, \epsilon_n} \Biggl[ \sup_{g \in \G} \frac{1}{n} \sum_{i=1}^n \epsilon_i g(Z_i) \Biggr] \right] .
\end{align*}
\end{definition}
In the above, we may for instance take the space $\Z = \X \times \Y$ and class $\G = \{\loss_f : f \in \F\}$.

The following uniform convergence result is well-known; a proof appears in \S~\ref{app:background}
for completeness.
\begin{lemma} \label{lemma:rademacher}
Let $(X_1, Y_1), \ldots, (X_n, Y_n)$ be independent samples from distribution $P$ and assume, for all $f \in \F$, that $\loss(f(X), Y) \in [0, B]$ almost surely. Further assume that the loss $\loss$ is $L$-Lipschitz in its first argument. 
With probability at least $1 - \delta$,
\begin{align*}
\sup_{f \in\F} \left\{ \E [ \loss(f(X),Y) ] - \frac{1}{n} \sum_{i=1}^n \loss(f(X_i),Y_i) \right\} 
\leq 2 L \rad_n(\F) + B \sqrt{\frac{\log(1/\delta)}{2 n}} .
\end{align*}
\end{lemma}
In particular, the upper bound holds for the empirical risk minimizer $\hat{f}$. A straightforward argument then leads to a high probability bound on the excess risk of $\hat{f}$ (the risk of $\hat{f}$ minus the risk of $f^*$).

\paragraph{ERM for multi-observation losses.}
To apply the above analysis to multi-observation loss functions, at first glance we would need samples of the form $(X,Y_1,\ldots,Y_m)$, with each $Y_i$ drawn i.i.d.\ from the conditional distribution $\D_X$.
In reality, however, we will likely have samples of the form $(X,Y)$ as above, and thus we cannot apply the traditional theory directly.
Nevertheless, there is hope to recover the kinds of guarantees of traditional ERM if nearby $x$ values have similar conditional distributions on $\Y$ (the domain of $Y$).
We will formalize this in \S~\ref{sec:risk-bounds} with a Lipschitz assumption on the total variation distance of the conditional distributions (see~\eqref{eqn:tv-lipschitz}).

\paragraph{Elicitation and ERM.}
How does the choice of loss function affect the resulting hypothesis $\hat{f}$ chosen by ERM?
One way to understand this relationship is to ask what statistic a loss function \emph{elicits}, meaning what minimizes the loss if given only a single $x$ value and a distribution over the $y$ values; specifically, what is $\argmin_r \E \loss(r,Y)$ as a function of the distribution of $Y$?
For example, it is well-known that squared loss elicits the mean $\E[Y]$ in the sense that $\E[Y] = \argmin_{r} \E (r - Y)^2$.
It can easily be shown that the behavior of ERM can be completely characterized by what statistic the loss elicits, provided that conditional statistic as a function of $x$ is in the function class $\F$~\citep{frongillo2016ec}.

Unfortunately, given a statistic, there do not always exist losses which elicit it; one such example is the variance~\citep{lambert2008eliciting}.
As a result, fitting a model to the conditional variances $f(x) \approx \Var[Y|X\tinyeq x]$ directly using ERM is impossible, and instead one would typically fit models to the first two moments $(\E[Y|X\tinyeq x],\E[Y^2|X\tinyeq x])$, from which the variance can be obtained.
This leads to the concept of \emph{elicitation complexity}~\citep{lambert2008eliciting,frongillo2015elicitation-2}: how many elicitable quantities does one need to obtain the desired statistic?
From previous work~\citep{casalaina-martin2017multi-observation}, we now know that allowing multiple observations can vastly reduce this complexity: with 2 observations, the variance has complexity 1 (see \S~\ref{sec:experiments}), and the complexity of the 2-norm of the distribution of $Y$ drops from $|\Y|-1$ to 1 again with only 2 observations.
This improvement in complexity can have real advantages in practice, as we show in \S~\ref{sec:experiments}, and is the main motivation for studying multi-observation ERM.

\section{Algorithms}
\label{sec:algorithms}

The setting we consider is slightly unusual in terms of the number and kinds of samples used.
Our algorithms will often take a number of different $(x,y)$ pairs and ``merge'' them to create a data point of the form $(x,y_1,\dots,y_m)$, i.e.~multiple observations associated to a single $x$.
We refer to such a tuple as a \emph{metasample}.
In this paper, $n$ will always denote the number of metasamples constructed and used by a particular algorithm.
This differs from $\ns$, the total number of data points drawn by the algorithm.
We focus on empirical risk minimization over the metasamples,
\begin{align*}
\argmin_{f\in\F} \sum_{i=1}^n \loss_f(x_i, y_{i,1}, \dots, y_{i,m}) .
\end{align*}
Therefore, the algorithmic questions are (1) how to draw or choose samples, and (2) how to construct metasamples.
(In \S~\ref{sec:risk-bounds}, we will prove theoretical risk guarantees for some of these algorithms.)

\paragraph{Learning paradigms.}
Our algorithms will apply in two different paradigms.
In \emph{supervised learning}, the algorithms draw $\ns$ data points of the form $(x,y)$ i.i.d., construct $n$ metasamples, and run ERM.
The \emph{sample complexity} is $\ns$.
In \emph{pool-based active learning}, the algorithms draw $\ns$ unlabeled $x$ points.
They may query up to one label $y$ for each $x$, drawn independently from $\D_x$.
This results in a smaller number of labeled pairs $(x,y)$, from which the algorithms construct the metasamples.
For those of our algorithms with theoretical guarantees, the \emph{label complexity} will always be $nm$, because every label we draw is used in exactly one metasample.\footnote{This is not always true of other algorithms described below and tested in simulations, where the number of metasamples is larger compared to the number of labels because each $y$ may appear in multiple metasamples.}

One paradigm we do not consider in this paper is (fully) \emph{active learning}, where algorithms may repeatedly choose any $x$, query it to obtain an independent draw $y \sim \D_x$, and repeat.
Here, traditional algorithms and guarantees will generally carry over to the multi-observation setting: One can generally query as many i.i.d. observations from $\D_x$ as desired, so multi-observation losses present no additional difficulty.
However, we note that multi-observation regression is well-motivated for active learning settings in practice, such as crowdsourcing settings where multiple crowd workers can label the same data point.

\paragraph{Algorithmic approach.}
Our algorithms attempt to mimic an idealized ERM in the following way: (1) draw $\ns$ data points; (2) choose $n$ ``clumps'' of data points; (3) pick a ``representative'' $x$ for each clump along with $m$ labels from each clump.
For our ``unbiased'' algorithms with theoretical guarantees, it is important that each representative $x$ be an i.i.d.\ draw from $\D$, so we first draw $n$ data points, then draw some number of additional data points so as to produce a ``clump'' around each $x$.
This naturally leads to a label complexity of $nm$ and a sample complexity of some $\ns$ depending on the algorithm.
For our other ``biased'' algorithms, we may have ``overlapping'' clumps.
This can have the advantage of $n$ becoming quite large, which may help ERM perform better in practice.
However, the metasamples become correlated because they share labels, making theoretical results challenging.

\subsection{Unbiased Algorithms}

We first consider algorithms for which we will later be able to prove risk bounds, by ensuring that the $x$ values in the metasamples are i.i.d.\ samples from $\D$.
The Na\"ive algorithm, Algorithm \ref{alg:naive}, starts by drawing $n$ i.i.d.\ data points $X_1^*,\dots,X_n^*$ and using them as the basis for a metasample.
For each $X_i^*$, it then draws many new data points, so that with high probability, enough points come close enough to form a good metasample.
It then moves on to the next $X_{i+1}^*$.
\vspace{-15pt}

\noindent
\begin{minipage}[t]{0.46\textwidth}
  \vspace{0pt}  
    \begin{algorithm}[H]
      \DontPrintSemicolon
      \KwIn{$n,m,N \in \mathbb{N}$}
      Sample $n$ points $X^*_1, \ldots, X^*_n$ indep. from $\D$\;
      
      \For{$i = 1 \algto n$}{
        Sample $k := N/n$ points $X^{(i)}_1, \ldots, X^{(i)}_k$ independently from $\D$\;
        
         \For{$j \in [m]$}{
           Set $X_{i,j}$ to the $j\nth$ nearest neighbor of $X^*_i$ among $(X^{(i)}_j)_{j \in [K]}$, with ties broken arbitrarily\;

           Sample a label $\tilde{Y}_{i,j} \sim \D_{X_{i,j}}$\;
         }
       }
       $\hat{f} = \mathrm{ERM}_{\mathcal{F},\loss}\left((X^*_i, (\tilde{Y}_{i,1}, \ldots, \tilde{Y}_{i,m}))_{i \in [n]}\right)$\;
       
       \Return $\hat{f}$\;
      \caption{\label{alg:naive} \textsc{Na\"ive Sampling}}
    \end{algorithm}
\end{minipage}%
\hfill
\begin{minipage}[t]{0.53\textwidth}
  \vspace{0pt}
    \begin{algorithm}[H]
      \DontPrintSemicolon
      \KwIn{$n,m,\ns \in \mathbb{N},\epsilon \in (0,1)$}
      Sample $n$ points $X^*_1, \ldots, X^*_n$ independently from $\D$\;
      
        \For{$j=1 \algto m$}{
        	Sample $k := \ns / m$ points $X^{(j)}_1, \ldots, X^{(j)}_k$ independently from $\D$.\;

        	Find a maximum matching $M^{(j)}$ between $X^*_1, \ldots, X^*_n$ and $X^{(j)}_1, \ldots, X^{(j)}_k$ where $X^*_i$ and $X^{(j)}_{i'}$ are adjacent iff $|X^*_i-X^{(j)}_{i'}| \leq \epsilon$\;

        	If $|M^{(j)}| < n$, arbitrarily match the remaining $X^*_i$'s (ignoring distance constraints)\;
            
        	\For{$i = 1 \algto n$}{
        		Let $X_{i,j}$ denote the match of $X^*_i$ in $M^{(j)}$\;

                        Sample a label $\tilde{Y}_{i,j} \sim \D_{X_{i,j}}$\;
        	}
        }
      \Return $\hat{f} = \mathrm{ERM}_{\mathcal{F},\loss}\left((X^*_i, (\tilde{Y}_{i,1}, \ldots, \tilde{Y}_{i,m}))_{i \in [n]}\right)$\;
      \caption{\label{alg:improved} \textsc{Improved Sampling}}
    \end{algorithm}
\end{minipage}
\vspace{5pt}

In \S~\ref{app:sampling} we prove a lemma stating that if we set $\ns =  m n^{(d+3)/2} d^{d/2} \log \frac{2 m n}{\delta}$, then with probability at least $1-\delta$, most of the points $(X^*_j)_{j \in [n]}$ have their $m$ nearest neighbors all within a proximity of $\frac{1}{\sqrt{n}}$.  The proof partitions $\X$ into ``heavy'' and ``light'' cells according to the probability mass of each and uses Hoeffding's inequality to bound the number of $X^*_j$ points in light cells.

To improve on Na\"ive Sampling, Algorithm~\ref{alg:improved} iteratively draws batches of new samples and finds a globally good assignment to the original base sample.\footnote{Note that metasamples constructed conditional on $X$ still ensure that $Y$ is drawn independently from $\D_X$.}  The following lemma guarantees this algorithm's performance for $\ns = O\left(mn + m\left(\frac{\sqrt{d}}{\epsilon}\right)^d \log \frac{md^d}{\delta \epsilon^d}\right)$.

\newcounter{impsamplemcntr}
\setcounter{impsamplemcntr}{\value{theorem}}

\begin{restatable}[Improved Uniform Sampling Lemma]{lemma}{impsamplem}
\label{lem:improved-sampling}
Let $\epsilon \in (0,1)$, $d \in \N$  and let $x^*_1, \ldots, x^*_n, x_1, \ldots, x_N$ be drawn independently from $\D=\mathrm{Unif}([0,1]^d)$. If $N \geq C m \left(n + \left(\frac{\sqrt{d}}{\epsilon}\right)^d \left(\log \frac{m}{\delta}+d \log\frac{d}{\epsilon}\right)\right)$ for a universal constant $C$, then with probability at least $1 - \delta$, for each $j \in [n]$, there are at least $m$ points $x_{i_{j,1}}, \ldots, x_{i_{j,m}}$ satisfying $\|x^*_j - x_{i_{j,m}}\|_2 \leq \epsilon$, and all the $i_{1,1}, \ldots, i_{1,m}, \ldots, i_{n,1}, \ldots, i_{n,m} \in [N]$ are distinct.
\end{restatable}

The proof appears in \S~\ref{app:sampling}, the main idea being as follows.  We partition $[0,1]^d$ into regions of diameter $\epsilon$ and, through a Poissonization argument, show that with sufficient probability, each region has at least $m$ $x_i$'s for every $x_j^*$.  \S~\ref{app:sampling} also provides an analogous lemma for nonuniform distributions, where we set  $\ns = O\left(m n^{(d+1)/2} d^{d/2} \log \frac{m(nd)^d}{\delta}\right)$ and $\epsilon=1/\sqrt{n}$.

\subsection{Biased Algorithms}
We now briefly consider algorithms that we feel are likely to perform well in practice and indeed do so in our simulations. 
These algorithms construct a larger number of metasamples by reusing labels, thereby giving them access to more information but making theoretical guarantees very difficult. For simplicity, both algorithms are specified for single dimensional $\X$, but they can be generalized to higher dimensions at the expense of computational complexity.
The first algorithm, Sliding Window (Algorithm \ref{alg:sliding}), simply iterates from left to right over the $x$-values on the real line and creates a metasample from each group of $m$ adjacent points.\footnote{A similar algorithm, not formalized here, partitions the points into disjoint sets of adjacent $m$-tuples, treating each as a metasample. We found this algorithm equally difficult to analyze and less performant in practice.}
The second, dubbed ``$\epsilon$-Nearby'' (Algorithm \ref{alg:nearby}), sets a fixed upper distance limit $\epsilon$ and constructs a metasample from all $m$-tuples of data points whose $x$-values lie in an interval of diameter $\epsilon$.
We demonstrate the performance of these algorithms in \S~\ref{sec:experiments}.

\noindent
\begin{minipage}[t]{0.47\textwidth}
  \vspace{0pt}
  \begin{algorithm}[H]
    \DontPrintSemicolon
    \KwIn{$n,m \in \mathbb{N}$}
    Sample $n$ pairs $(X_1,Y_1), \ldots, (X_n,Y_n)$\;
    
    Sort the pairs so that $X_1 \leq X_2 \ldots \leq X_n$\;
   			
    \For{$i = 1 \algto n-m+1$}{
      Set $\tilde{X}_i \leftarrow \frac{1}{m}\sum_{j=1}^m X_{i+j-1}$\;

      \For{$j = 1 \algto m$} {
        Set $\tilde{Y}_{i,j} \leftarrow Y_{i+j-1}$\;
      }
    }
    $\hat{f} = \mathrm{ERM}_{\mathcal{F},\loss}((\tilde{X}_i, (\tilde{Y}_{i,1}, \ldots, \tilde{Y}_{i,m}))_{i=1}^{n-m+1})$\;

    \Return $\hat{f}$\;
    \caption{\label{alg:sliding} \textsc{Sliding Window}}
  \end{algorithm}
\end{minipage}%
\hfill
\begin{minipage}[t]{0.52\textwidth}
  \vspace{0pt}
  \begin{algorithm}[H]
    \DontPrintSemicolon
    \KwIn{$n,m \in \mathbb{N}, \epsilon \in (0, 1)$}
    Sample $n$ pairs $(X_1,Y_1), \ldots, (X_n,Y_n)$\;
    
    Sort the pairs so that $X_1 \leq X_2 \ldots \leq X_n$\;
    
    Set $t \leftarrow 1$\;
    
    \For{$i_1 = 1 \algto n-m+1$}{
      Set $k \gets \max \{ j \mid i_1 < j \leq n \wedge X_j - X_{i_1} \leq \epsilon\}$\;
      \For{$\{i_2,\ldots,i_m\} \subseteq \{i_1+1,\ldots,k\}$}{
        Set $\tilde{X}_t \leftarrow \frac{1}{m}\sum_{j=1}^m X_{i_j}$\;
        
        Set $\tilde{Y}_{t,j} \leftarrow Y_{i_j}$\;
        
        \lFor{$j = 1 \algto m$}{
          Set $t \leftarrow t+1$\;
        }
      }
    }
    \Return $\hat{f} = \mathrm{ERM}_{\mathcal{F},\loss}((\tilde{X}_i, (\tilde{Y}_{i,1}, \ldots, \tilde{Y}_{i,m}))_{i \in [t]})$\;
    \caption{\label{alg:nearby} \textsc{$\epsilon$-Nearby}}
  \end{algorithm}
\end{minipage}

\section{Risk bounds}
\label{sec:risk-bounds}

Let $\D$ be an arbitrary probability distribution over $\X \subset \reals^d$, and, for each $x \in \X$, let $\D_x$ be the conditional distribution over $\Y$ given $x$. 
For simplicity, we take $\X = [0, 1]^d$. 
In this section, we assume that loss values lie in $[0, B]$.

As discussed above, to make headway we will need to relate the conditional distributions $D_x$ for nearby $x$; formally, we make a Lipschitz assumption on their total variation distance,
\vspace{-1mm}
\begin{align}
D_{TV}(\D_x, \D_{x'}) \leq K \|x - x'\|_2 . \tag{A1} \label{eqn:tv-lipschitz}
\end{align}
Intuitively, this means that samples from nearby conditional distributions are almost interchangeable.

\subsection{Excess risk bounds for general situations}

Consider an ideal setting where we are given an i.i.d.~sample of $n$ points $X^*_1, \ldots, X^*_n$ and, for each $i \in [n]$, we have an i.i.d.~sample of $m$ points $Y_{i,1}, \ldots, Y_{i,m}$ drawn from $\D_{X^*_i}$.
Here, existing analyses of ERM still go through, modulo the label for each $X^*_i$ now being $\psi(Y_{i,1}, \ldots, Y_{i,m})$ for some known function $\psi$.
Our analysis relates the performance of our algorithms, which must construct their own ``noisy'' metasamples from imperfect data, to this ideal.
The analysis readily provides results under both the standard notion of sample complexity from supervised learning and the (much smaller) label complexity from pool-based active learning.

A key idea in the analysis is to view each metasample as being drawn in this idealized fashion (i.e.~each $Y_{i,j} \sim \D_{X^*_i}$), but with some chance of corruption.
We show this is possible by viewing any nearby distribution $\D_{X'}$ as a mixture of $\D_{X^*_i}$ with an arbitrary corruption.
We can then analyze ERM on a set of metasamples, most of which are ideal, but some of which are corrupted.

\begin{theorem}[Excess risk with corrupted samples] \label{thm:bound-imperfect-sampling} 
Assume that \eqref{eqn:tv-lipschitz} holds. Let \\$\ns = C m n^{(d+1)/2} d^{d/2} \log \frac{m(nd)^d}{\delta}$ for some universal constant $C$, and let $\tilde{f}$ be the hypothesis returned by Algorithm~\ref{alg:improved} on $(n, m, \ns, 1/\sqrt{n})$.
Then, for $n \geq 2 \log \frac{8}{\delta}$, with probability at least $1 - \delta$,
\begin{align*}
\E [ \loss_{\tilde{f}}(X, \mathbf{Y}) ] 
\leq \E [ \loss_{f^*}(X, \mathbf{Y}) ] 
       + 2 L \rad_{n}(\F) 
       + 2 B \left( 2 \sqrt{\log \tfrac{4}{\delta}} + m K \right)\frac{1}{\sqrt{n}} ~,
\end{align*}
where $X$ is drawn from $\D$, and, conditionally on $X$, $\mathbf{Y} = (Y_1, \ldots, Y_m)$ is drawn from $(\D_X)^m$.
\end{theorem}
\begin{Proof}[Proof Sketch]
First, we appeal to Algorithm~\ref{alg:improved} to obtain $n$ metasamples $\{(X_i, \tilde{Y}_{i,1}, \dots, \tilde{Y}_{i,m})\}_{i=1}^n$ with each $X_i$ an i.i.d.~draw from $\D$ and each $\tilde{Y}_{i,j}$ an independent draw from some $\D_{X_i'}$ with $\|X_i' - X_i\|_2 \leq \frac{1}{\sqrt{n}}$.
This holds except for $O(\sqrt{n})$ metasamples.

Now, from Assumption \eqref{eqn:tv-lipschitz} we have, for each $X_i$, $m$ independent samples $\tilde{Y}_{i,j}$ from distributions that are close to $\D_{X_i}$.
The key idea is to show that each metasample $i$'s labels can be viewed as coming from $\D_{X_i}^m$ with high probability and from an arbitrary distribution otherwise.
This argument is first made for each $\tilde{Y}_{i,j}$: We can view it as a sample from a mixture that puts high probability on $\D_{X_i}$ and small probability on some other distribution.
Under this view, with high probability, every $\tilde{Y}_{i,1}, \dots, \tilde{Y}_{i,m}$ comes from the $\D_{X_i}$ component of its mixture.
Of course, this fails to be true for some of the metasamples, which we show again number only $O(\sqrt{n})$ with high probability.

The final component is an analysis of \emph{ERM with corrupted samples}.
Consider (even in the classical setting) running ERM on a set of $n$ samples, of which $O(\sqrt{n})$ have been corrupted arbitrarily but the rest are drawn i.i.d.\ from the underlying distribution.
In this case, we show that standard generalization bounds continue to hold with an error loss of only $O(1/\sqrt{n})$.
\end{Proof}

Recall that $nm$ is the label complexity (pool-based active learning paradigm) and $\ns$ is the sample complexity (supervised learning paradigm).\footnote{The algorithm is the same in both paradigms, except for the timing of label draws: in supervised, labels are drawn together with each $X$, while in pool-based active learning, labels are only queried when $X$ is added to a metasample.}
To illustrate exactly how our results translate, let us adopt the parametric setting where the Rademacher complexity term $\rad_n(\F)$ decays at the rate of $O(1/\sqrt{n})$ with $n$ (meta)samples. Because the sample complexity is $\ns =  C m n^{(d+1)/2} d^{d/2} \log \frac{m(nd)^d}{\delta}$, the excess risk decays as the rate $\tilde{O}(1/N^{1/(d+1)})$, where the $\tilde{O}$ notation omits log factors.
Similarly, in pool-based active learning, we can write $n' = nm$ for the label complexity and get excess risk decaying at a rate $O(\sqrt{1/n'})$.
(We fix $m$ here as it is inherent to the loss function.)

\subsection{Faster rates under strong convexity and the uniform distribution}

Let $\F$ be a class of linear predictors, so that it can be identified with a set $\W \subset \reals^d$. For $w \in \W$, and fixed $(x, \mathbf{y}) \in \X \times \Y^m$, we assume that the loss has the generalized linear form
$\loss \colon w \mapsto c(\langle w, \phi(x) \rangle, \mathbf{y})$,
for some functions $c$ and $\phi$. 
The risk functional $R$ is then
\begin{align*}
R \colon w \mapsto \E \left[ c \bigl( \langle w, \phi(X) \rangle, \mathbf{Y} \bigr) \right] ,  
\end{align*}
where $X \sim \D$ and, conditionally on $X$, $\mathbf{Y} = (Y_1, \ldots, Y_m) \sim \left( \D_X \right)^m$.

\begin{theorem} \label{thm:bound-faster} 
Assume that \eqref{eqn:tv-lipschitz} holds. 
Let $\F$ be a class of linear functionals as above, with the loss taking the generalized linear form. Suppose that $\|\phi(x)\| \leq B$ (the same $B$ as for the upper bound on the loss). 
Let $\varepsilon = (m K n)^{-1}$ and $\ns =C m \left(n + (\frac{\sqrt{d}}{\epsilon})^d \log \frac{3md^d}{\delta\epsilon^d}\right)$ for a universal constant $C$, and let $\tilde{f}$ be the hypothesis returned by Algorithm~\ref{alg:improved} on $(n,m,\ns,\epsilon)$. 
If $\D$ is the uniform distribution over $[0, 1]$ and the risk functional $R$ is $\sigma$-strongly convex, then, for any $\delta \leq 3 e^{-4}$ and $n \geq 2 \log \frac{8}{\delta}$, 
with probability at least $1 - \delta$
\begin{align*}
\E [ \loss_{\tilde{f}}(X, \mathbf{Y}) ] 
\leq \E [ \loss_{f^*}(X, \mathbf{Y}) ] 
       + \frac{3 B \log \frac{3}{\delta}}{n} 
       + \frac{8 L^2 B^2 \left( 32 + \log \frac{3}{\delta} \right)}{\sigma n}
\end{align*}
where $X$ is drawn from $\D$, and, conditionally on $X$, $\mathbf{Y} = (Y_1, \ldots, Y_m)$ is drawn from $(\D_X)^m$.
\end{theorem}

This result implies that the excess risk decays at the rate $O(1/n)$, where $nm$ is the label complexity.  Since the sample complexity is $\ns = C m \left(n + (mKn\sqrt{d})^d \log \frac{3m(mKnd)^d}{\delta}\right)$, this implies that the excess risk decay rate in terms of $\ns$ is $\tilde{O}(1/N^{1/d})$.

\begin{example}[Strongly convex risk functional] \label{ex:strong-convexity}
Take the example of the variance with $\X \subset \reals^d$. 
For fixed $(x, y_1, y_2)$, the loss and risk functional are, respectively,
\begin{equation*}
\loss \colon w \mapsto \left( \langle w, x \rangle - \frac{1}{2} (y_1 - y_2)^2 \right)^2~, \qquad R \colon w \mapsto \E_{\substack{X \sim \mathcal{D} \\ Y_1 \sim \mathcal{D}_X \\ Y_2 \sim \mathcal{D}_X}} \left[ \left( \langle w, X \rangle - \frac{1}{2} (Y_1 - Y_2)^2 \right)^2 \right]~.
\end{equation*}
Then $\nabla^2_w R(w) = 2 \E [ X X^T ]$, and so if $\E [ X X^T ] \succeq \sigma I$, then $(2 \sigma)$-strong convexity holds.
In the special case of $d = 1$, provided that $X$ is non-trivial we clearly have strong convexity of the risk.
\end{example}

\subsection{Sample complexity and dimension}
\label{sec:dimension}

While our sample complexity results above apply for any dimension $d$ of $\X$, they scale exponentially in $d$.
In this work, we are motivated by $d = O(1)$ and show in the next section that the multi-observation approach can yield significant practical improvements in these settings.
Here, we briefly note that an exponential dependence on $d$ is information-theoretically necessary for \emph{any algorithm} to regress on such higher-order statistics, even, in particular, in the simple case of estimating the average $\Var(y \given x)$ over the distribution.
Therefore, we leave investigation of strategies for the higher-dimensional case (for example, active learning approaches) to future work.

For example, in \S~\ref{sec:lower} we show:
\newcounter{booldimlowercntr}
\setcounter{booldimlowercntr}{\value{theorem}}
\begin{restatable}{theorem}{booldimlower}
\label{thm:boolean-dimension-lower}
 If $\X$ is in the $d$-dimensional hypercube and the Lipschitz constant is $K = 1$, no algorithm for regression on variance of $y$ can have nontrivial accuracy with $o(2^{0.5d})$ samples.
\end{restatable}
Intuitively, the obstacle is that a subexponential number of samples can (and, under e.g.~a uniform distribution, does) have all $x$ values separated from each other by a constant distance.
So any given region will only have one $(x,y)$ pair sampled, and information-theoretically, no knowledge can be gleaned about the variance in that region.
\S~\ref{sec:lower} also provides a similar lower bound for the case of the uniform distribution over $[0,1]^d$, where $K = d$.

\section{Comparison to Single-Observation Losses}

For completeness, we now compare multi-observation regression to the typical case of single-observation regression, where the loss depends on a single prediction and a single observation.
We first explore this comparison theoretically, and then give experimental results, both in the simple setting of fitting a parametric model to the variance $\Var[Y|X]$.
In our experiments, we also take the opportunity to compare the various algorithms presented in \S~\ref{sec:risk-bounds} and explore other statistics of interest besides variance.

\subsection{Lower bounds for single-observation losses}
\label{sec:stone-lower-bound}

Consider the example of the variance. 
One method for regressing the variance is the following ``two-estimator'' approach: use a single-observation estimator to regress the first conditional moment $\E [ Y \mid X ]$ and a separate single-observation estimator to regress the second conditional moment $\E [ Y^2 \mid X ]$; the variance can then be predicted in terms of the two learned hypotheses.

Intuitively, this approach can go very wrong in many settings.
Although the conditional variance might have a simple parametric form, we may not know a priori if these conditional moments have simple forms. Thus, we may either try to estimate them from classes that are too small, picking up a large approximation error that leads to poor estimators of the variance, or we may conservatively estimate them from a sufficiently rich class to match our Lipschitz assumption \eqref{eqn:tv-lipschitz}, which may lead to 
overfitting. 
Indeed, 
an existing minimax lower bound of \cite{stone1982optimal} for estimating the conditional mean $\E [ Y \mid X]$ when the latter is $1 / (2 \sqrt{d})$-Lipschitz with respect to $x$
(as per Assumption~\eqref{eqn:tv-lipschitz}) with range in $[0, 1]$ implies that \emph{the risk of estimating the variance (under squared loss) of any two-estimator approach $\hat{f}$,
\begin{align}
\E \left[ \left( \hat{f}(X) - \Var[Y \mid X] \right)^2 \right] , \label{eqn:1-step}
\end{align}
cannot decay at a rate faster than $n^{-2/(2 + d)}$.} 
Note that the estimation setting is only made easier by the fact that the Lipschitz constant decreases with $d$; this lower bound applies a fortiori for $K$-Lipschitz functions when $K$ does not decrease with the dimension. Lastly, since the bound is a minimax lower bound, it holds for all estimators, not just those that predict according to $K$-Lipschitz functions.

We now describe how the lower bound of \cite{stone1982optimal} implies the above result. 
As above, 
the goal is to minimize the $L_2(\D)$-norm of the conditional variance of $Y$, i.e., to find an estimator $\hat{f}$ for which the squared version thereof, \eqref{eqn:1-step}, 
is as small as possible.
In the two-estimator approach, $\hat{f}$ takes the form $\hat{f} = \hat{g} - \hat{h}$ for estimators $\hat{g}$ and $\hat{h}$ of $\E [ Y^2 \mid X ]$ and $\left( \E [ Y \mid X] \right)^2$ respectively. Then \eqref{eqn:1-step} is equal to 
\begin{align}
\E \left[ \left( \left( \hat{g}(X) - \E \left[ Y^2 \mid X \right] \right) + \left( \E \left[ Y \mid X \right] - \hat{h}(x) \right) \right)^2 \right] . \label{eqn:2-step}
\end{align}
The two-estimator approach needs to ensure that $L_2(\D)$ norm of $\hat{g}(X) - \E \left[ Y^2 \mid X \right]$ is small (and that the $L_2(\D)$ norm of the other term is small). Suppose that $Y \mid X = x$ is Bernoulli for any $x \in \X$. Then this latter necessary goal reduces to the familiar regression problem of minimizing
\begin{align*}
\E \left[ \left( \hat{g}(X) - \E [ Y \mid X] \right)^2 \right] .
\end{align*}

The aforementioned minimax lower bound will apply to the above estimation problem in the following setting (we leave the full details to \S~\ref{app:lipschitz-lower-bound}):  take $\X = [0, 1]^d$ and let $\D$ be the uniform distribution over $\X$. Suppose for all $x \in \X$ that the distribution $\D_x$ is a certain subclass of Bernoulli distributions
with $x \mapsto \E [ Y \mid X = x ]$ a $K$-Lipschitz function. Then for any estimator $\hat{g}$, there exists a law $Y \mid X$ satisfying the aforementioned assumptions such that
\begin{align*}
  \E \left[ ( \hat{g}(X) - \E [ Y \mid X ] )^2 \right] = \Omega \left( n^{-2 / (2 + d)} \right) ,
\end{align*}
and there is a matching upper bound, so that this is the minimax optimal rate of convergence. 

To see how this rate compares to our results, suppose that the variance is captured by a generalized linear form which lies within our model $\F$, so that $f^* \colon x \mapsto = \Var [ Y \mid X = x]$. Then a slight generalization of a standard exercise shows that the excess risk $\E [ \loss_{\tilde{f}}(X, \mathbf{Y}) ] - \E [ \loss_{f^*}(X, \mathbf{Y}) ]$ in Theorem~\ref{thm:bound-faster} takes the form \eqref{eqn:2-step} with $\hat{f}$ replaced by our direct, metasample-based estimator $\tilde{f}$. In the case of dimension $d=1$, the rate of Theorem~\ref{thm:bound-faster} is $n^{-1}$ and hence better than the rate of $n^{-2/3}$ of the two-estimator approach. Note that there is no contradiction with the minimax lower bound of \cite{stone1982optimal}, because we assume that the variance takes a parametric form; were the variance to be an arbitrary $K$-Lipschitz function, our rates would degrade.\footnote{We believe in this case that our rate for $d=1$ would also be $n^{-2/3}$, based on existing fast rates results for classes whose metric entropy grows as $\varepsilon^{-1}$ (the class of $K$-Lipschitz functions exhibits such growth).}

One might wonder if it is possible to beat the above lower bound by using active learning. Unfortunately, the same lower bound holds in a similar setting even if one is allowed to use active learning strategies~\cite[Theorem 1]{castro2006faster}.

In our experiments, we explore the performance of the two-estimator approach using ERM, which can be significantly worse than a direct multi-observation regression of the variance.

\subsection{Experiments}
\label{sec:experiments}

In our experiments we opted for synthetic data over real data.  Multi-observation loss functions help in learning higher order statistics about $Y | X\tinyeq x$; unfortunately with real data one generally does not know what the true value of those statistics are and thus has no objective way of comparing different algorithms.  By using synthetic data we can choose the underlying value of these statistics and evaluate algorithms by how closely they approximate it.

We consider three different statistics: variance, the upper confidence bound (UCB), and MINVAR.
The variance can be elicited, as observed in prior work~\citep{casalaina-martin2017multi-observation}, by the two-observation loss function $\loss(r,y_1,y_2) = \left(r-\frac{1}{2}(y_1-y_2)^2\right)^2$.
The UCB, often used in surrogate optimization for robust design in engineering, is defined by $\ucb_\lambda(Y) = \E[Y] + \lambda \sqrt{\Var[Y]}$, for some fixed $\lambda$.
We show in \S~\ref{app:ucb} that it can be elicited by a two-observation loss function under some restrictions on the distribution of $Y$.
MINVAR ($\rho_k$), the expected minimum of $k+1$ i.i.d. draws of the distribution, is straightforwardly elicitable via squared loss $\loss(r,y_1,\ldots,y_{k+1}) = \left(r-\min\{y_1,\ldots,y_{k+1}\}\right)^2$.

In all of our experiments we are trying to learn a statistic of $Y | X\tinyeq x$, where $X \sim \mathrm{Unif}(0,1)$.  For the variance, we tried different distributions $Y | X\tinyeq x$ of the form $f(x) + N(0,1)$.  We present here our results when $f(x)$ is either a sine wave or a line. $\Var(Y|X\tinyeq x) = 1$ in all cases.  For the $\ucb_\lambda(Y)$ experiments, we chose $\lambda = 8$ and $Y | X\tinyeq x \sim \Gamma(k(x), \theta(x))$ where $k(x)$ and $\theta(x)$ were chosen such that $\E[Y|X\tinyeq x] = 2 + \sin(4\pi x)$ and $\ucb_\lambda(Y|X\tinyeq x) = x+10$.  For the $\rho_k$ experiments we let $k = 4$ and $Y | X\tinyeq x$ be an exponential distribution with mean $5(x+2)$.  Consequently, $\rho_4(Y | X\tinyeq x) = x+2$.

As a baseline for the variance and $\ucb_\lambda$, we compared our algorithms to the strategy of learning functions for $E[Y|X\tinyeq x]$ and $E[Y^2|X\tinyeq x]$ via ERM and then combining those functions to estimate $\Var[Y|X\tinyeq x]$ or $\ucb_\lambda(Y|X\tinyeq x)$.  For $\rho_k$ there is no such obvious single observation ERM strategy.  For $\ucb_\lambda$ and $\rho_k$ we also compared our algorithms against the ``empirical approach'' that is given the power to draw multiple i.i.d. labels for each random $x$ it draws, compute the empirical statistic for that $x$, then fit a line to the results.  See \S~\ref{app:sims} for more details about our experiments.

\begin{figure}[h]
	\centering
		\includegraphics[width=.49\textwidth]{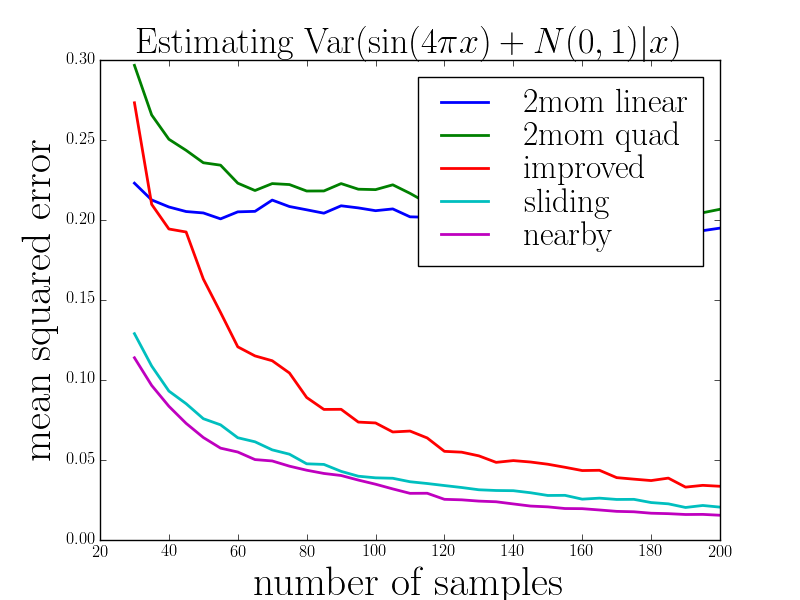}
		\includegraphics[width=.49\textwidth]{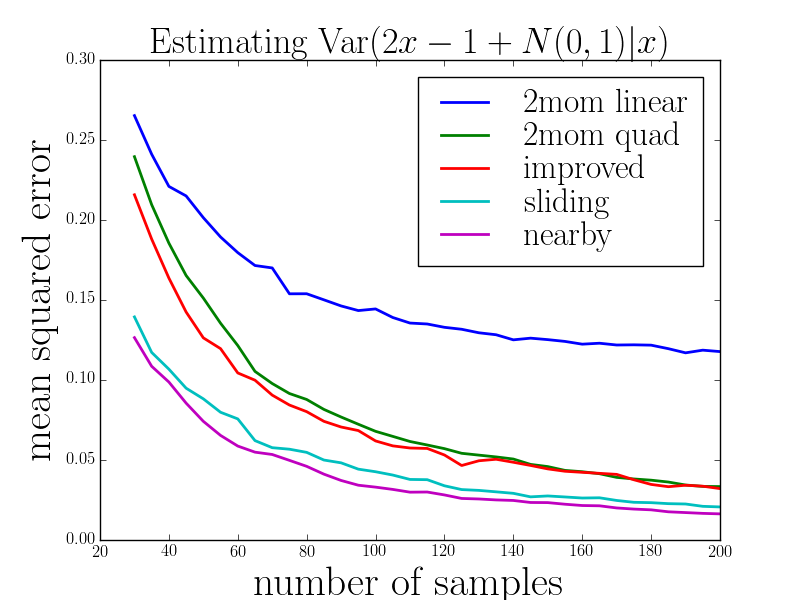}
		\includegraphics[width=.49\textwidth]{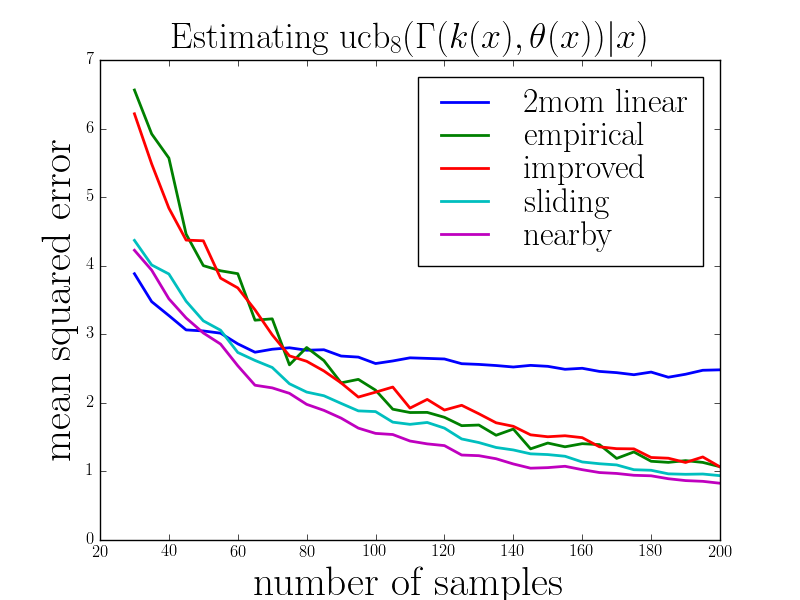}
		\includegraphics[width=.49\textwidth]{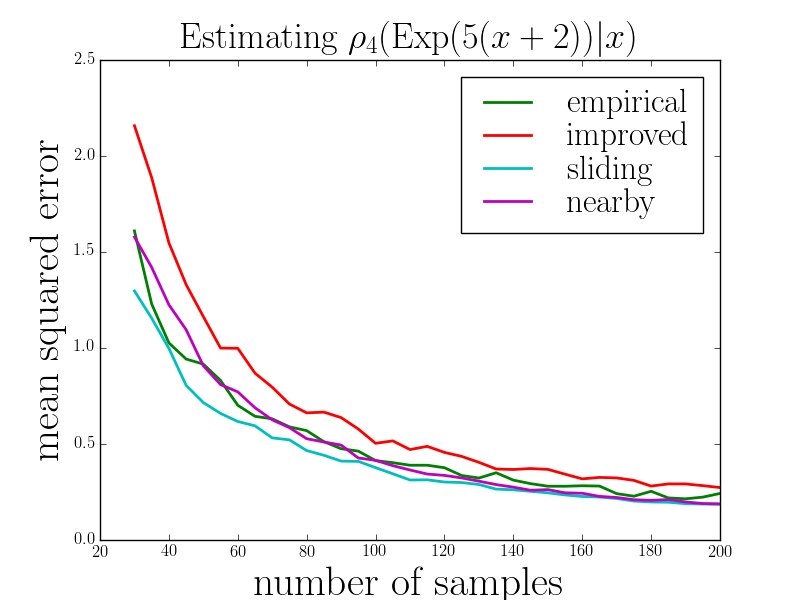}
	\caption{A comparison of ERM strategies.  Here ``2mom linear'' and ``2mom quad'' are the methods of fitting lines and quadratics respectively to the first two moments of $Y | X\tinyeq x$, ``improved'' is Algorithm~\ref{alg:improved}, ``sliding'' is Algorithm~\ref{alg:sliding}, ``nearby'' is Algorithm~\ref{alg:nearby} and ``empirical'' is the strategy of fitting a line to empirical estimates of the given statistic. Methods are evaluated by mean squared error to the true value of the statistic.}
	\label{fig:sims}
\end{figure}

Our results are depicted in Figure~\ref{fig:sims}.  Observe that in all of our experiments, Algorithms~\ref{alg:sliding} and~\ref{alg:nearby} performed the best.  Fitting lines to the moments never performed well, which is not surprising as in all cases at least one moment was non-linear.\footnote{We found that reducing the magnitude of the ``non-linear part'' led the two-moment approach to outperform our algorithms.  This occurred for variance and $\ucb_\lambda$ when $|f(x)|$ and $\lambda$ were sufficiently small, respectively.}  However, even when fitting quadratics to moments which are quadratics (in the case of the variance when $f(x) = 2x-1$) our two observation algorithms still outperformed the two moment approach.  This demonstrates that for the two observation approach to be beneficial it is only necessary that the statistic is simpler than the underlying distribution, not that the underlying distribution is from an entirely unknown class.  Our algorithms only show very slight improvement over the empirical approach for $\ucb_\lambda$ and $\rho_k$, but even being competitive is surprising given that the empirical approach has more power than our algorithms are afforded.

\section{Conclusion and Future Work}
\label{sec:conclusion}

ERM with multi-observation loss functions presents challenges and complications as compared to traditional ERM, but also interesting opportunities.
With initial theoretical guarantees in hand, one next step is to explore  ``risk'' or ``variance'' regression problems encountered in practice for which the multi-observation approach may be useful.
Active learning settings may be among the most fruitful, as they are well-suited to collecting multiple labels at or near the same feature.
This investigation will interact with \emph{elicitation} and the design of loss functions that are consistent for a desired property of the conditional distribution, such as the variance.
In particular, an open problem is to discover a loss function for the UCB property that does not require restrictions on the distribution of $Y$.

Another direction is to investigate the higher-dimensional setting further.
Our lower bounds show that, without additional assumptions, the number of samples required for regression on e.g.~the variance require sample complexity exponential in the dimension of $\X$.
One approach could be investigation into settings with low intrinsic dimension, embedded in some higher-dimensional space.
Another could be investigating active learning applications, which should allow sidestepping these lower bounds.


\bibliography{diss,mult-reg-refs}

\newpage
\appendix

\section{Proofs from background}
\label{app:background}

\begin{Proof}[Proof of Lemma~\ref{lemma:rademacher}]

Let $\Prob$ be the probability measure operator with respect to $P$, and let $\Pn$ be the empirical measure operator with respect to the sample $(X_1, Y_1), \ldots, (X_n, Y_n)$; that is, for each $f \in \F$, we have $\Prob \loss_f = \E [ \loss_f(X, Y) ]$ and $\Pn \loss_f = \frac{1}{n} \sum_{i=1}^n \loss_f(X_i, Y_i)$.

From the bounded differences inequality (Lemma 1.2 of \cite{mcdiarmid1989method}) with bounded differences $c_i = \frac{B}{n}$,
\begin{align*}
\Pr \left( \sup_{f \in \F} (\Prob - \Pn) \loss_f > \E \left[ \sup_{f \in \F} (\Prob - \Pn) \loss_f \right] + t  \right) \leq e^{-2 n t^2 / B^2} .
\end{align*}

Next, let $(X'_i, Y'_i)_{i \in [n]}$ be an independent copy of $(X_i, Y_i)_{i \in [n]}$. 
Jensen's inequality implies that
\begin{align*}
\E \left[ \sup_{f \in \F} (\Prob - \Pn) \loss_f \right] 
&= \E \left[ \sup_{f \in \F} \frac{1}{n} \sum_{i=1}^n \left( \E [ \loss_f(X'_i,Y'_i) ] - \loss_f(X_i,Y_i) \right) \right] \\
&\leq \E \left[ \sup_{f \in \F} \frac{1}{n} \sum_{i=1}^n \left( \loss_f(X'_i,Y'_i) - \loss_f(X_i,Y_i) \right) \right] ,
\end{align*}
which, for independent Rademacher random variables $\epsilon_1, \ldots, \epsilon_n$ is in turn equal to
\begin{align*}
\E \left[ \E_{\epsilon_1, \ldots, \epsilon_n} \left[ \sup_{f \in \F} \frac{1}{n} \sum_{i=1}^n \epsilon_i \left( \loss_f(X'_i,Y'_i) - \loss_f(X_i,Y_i) \right) \right] \right] 
&\leq 2 \E \left[ \E_{\epsilon_1, \ldots, \epsilon_n} \left[ \sup_{f \in \F} \frac{1}{n} \sum_{i=1}^n \epsilon_i \loss_f(X_i,Y_i) \right] \right] \\
&= 2 \rad_n(\{ \loss_f \colon f \in \F \} ) .
\end{align*}
The result follows from Theorem 7 of \cite{meir2003generalization}, which shows that under the $L$-Lipschitzness of the loss as a function of the prediction $\hat{y} = f(x)$, we have the comparison inequality $\rad_n(\{ \loss_f \colon f \in \F \} )  \leq L \rad_n(\F)$.
\end{Proof}

\section{Proofs of sampling methods}
\label{app:sampling}

\begin{lemma}[Na\"ive Sampling Lemma]
\label{lemma:sampling}
Let $d \in \N$ and let $x^*_1, \ldots, x^*_n, x_1, \ldots, x_N$ be drawn independently from a distribution $\D$ on $\X = [0,1]^d$. If $N \geq m n^{(d+3)/2} d^{d/2} \log \frac{2 m n}{\delta}$, then with probability at least $1 - \delta$, there is a set $\mathcal{J}$ of cardinality at least $n - \sqrt{(n \log \frac{2}{\delta}) / 2}$ for which, for each $x^*_j$ with $j \in \mathcal{J}$, there are at least $m$ points $x_{i_{j,1}}, \ldots, x_{i_{j,m}}$ satisfying $\|x^*_j - x_{i_{j,m}}\|_2 \leq \frac{1}{\sqrt{n}}$, and all the $i_{1,1}, \ldots, i_{1,m}, \ldots, i_{|\mathcal{J}|,1}, \ldots, i_{|\mathcal{J}|,m} \in [N]$ are distinct.
\end{lemma}
\begin{proof}
Take $C_1, \ldots, C_r$ to be a partition of $\X$ for which every cell $C_j$ has diameter $\sup_{x,x' \in C_j} \|x - x'\|_2$ at most $\varepsilon$. Observe that we may always take $r \leq \N(\X, \varepsilon/2)$, where $\N(\X, \varepsilon) $ is the minimum number of radius-$\varepsilon$ balls in the Euclidean norm $\|\cdot\|_2$ whose union contains $\X$.  Note that $r \leq (\sqrt{d}/\epsilon)^d$.

We take $\varepsilon = \frac{1}{\sqrt{n}}$ 
and partition the set of cells of $\X$ into light cells and heavy cells, where any light cell $C_j$ satisfies $\Pr(C_j) \leq \frac{1}{r\sqrt{n}}$. Since there are at most $r$ cells, the aggregate probability measure among all the light cells is at most $\frac{1}{\sqrt{n}}$. Hoeffding's inequality implies that only with probability at most $\delta/2$ will more than 
$\sqrt{(n \log \frac{2}{\delta}) / 2}$ 
samples of $x_1^*, \ldots, x_n^*$ fall into light cells. 
The remainder of the points therefore fall into heavy cells, each of which has probability measure at least $\frac{1}{r\sqrt{n}}$. Now, for some fixed $x_j^*$ in a heavy cell, if we sample $r\sqrt{n} \log \frac{1}{\delta}$ points, then with probability at least $1 - \delta$ at least one of these latter points would fall into the same cell as $x_j^*$. Thus, if we sample 
$$N = m n^{3/2}r \log \frac{2 m n}{\delta} \leq m n^{(d+3)/2} d^{d/2} \log \frac{2 m n}{\delta}$$ 
points $x_1, \ldots, x_N$, then with probability at least $1 - \delta/2$ every $x_j^*$ in a heavy cell will have at least $m$ samples falling into its cell.
\end{proof}

\newcounter{temptheorem}
\setcounter{temptheorem}{\value{theorem}}
\setcounter{theorem}{\value{impsamplemcntr}}
\impsamplem*
\setcounter{theorem}{\value{temptheorem}}
\begin{proof}
We will prove this for the case when $m = 1$, from which the general result follows by setting $\delta = \delta / m$ and repeating $m$ times.

We partition $[0,1]^d$ into $r=(\sqrt{d}/\epsilon)^d$ hypercubes of width $w=\epsilon / \sqrt{d}$, $C_1,\ldots,C_r$.  $S_i = |\{i | x_i^* \in C_i\}|$ is the number of $x_i^*$'s that lie in the $i\nth$ hypercube.  Similarly, 
let $T_i = |\{i | x_i \in C_i\}|$ be the number of $x_i$'s that fall in the $i\nth$ hypercube.  We will now show that for all $i \in [r]$, $T_i \geq S_i$ with probability at least $1-\delta$.  The lemma follows from this as we can then simply match up points within each hypercube and the maximum distance between two points in a hypercube is $w\sqrt{d}=\epsilon$.

In order to bound $\Pr[\forall i \; T_i \geq S_i]$, we will first consider a slight alteration of our setting.  Instead of having fixed numbers of samples $n$ and $N$, they will be random variables $\tilde{n} \sim \mathrm{Pois}(2n)$ and $\tilde{N} \sim \mathrm{Pois}(N/2)$.  In this setting, we let $\tilde{S_i}$ and $\tilde{T_i}$ be the number of samples of each category falling into the $i\nth$ interval.  The key property we will use here is that 
$$\tilde{S_i} \sim \mathrm{Bin}(\tilde{n},w^d) = \mathrm{Pois}(2nw^d),$$ 
and analogously $\tilde{T_i} \sim \mathrm{Pois}(Nw^d/2)$.

Using this we have
\begin{align*}
\Pr[\tilde{T_i} < \tilde{S_i}] &= \Pr[\mathrm{Pois}(N w^d/2) < \mathrm{Pois}(2n w^d)]\\
&\leq e^{-\left(\sqrt{Nw^d/2}-\sqrt{2nw^d}\right)^2}
\end{align*}
where the final inequality is a standard bound on Poisson races which follows from a Chernoff bound \cite[Appendix A]{kamath2015optimal}.  By a union bound, the probability that $\tilde{T_i} \geq \tilde{S_i}$ for all $i \in [r]$ is lower bounded by
$$1 - r \cdot e^{-\left(\sqrt{Nw^d/2}-\sqrt{2nw^d}\right)^2}.$$

Now observe that
\begin{align*}
\Pr[\forall i \; T_i \geq S_i] &= \Pr[\forall i \; \tilde{T_i} \geq \tilde{S_i} \mid \tilde{N} = N \wedge \tilde{n} = n] \\
&\geq \Pr[\forall i \; \tilde{T_i} \geq \tilde{S_i} \mid \tilde{N} \leq N \wedge \tilde{n} \geq n]
\end{align*}
because increasing $N$ or decreasing $n$ only increases our chances of finding a matching.  Let $I$ be the event that $\tilde{N} \leq N \wedge \tilde{n} \geq n$, and let $\bar{I}$ be its negation.
\begin{align*}
\Pr[\forall i \; \tilde{T_i} \geq \tilde{S_i} \mid I] &= \frac{\Pr[\forall i \; \tilde{T_i} \geq \tilde{S_i}] - \Pr[\forall i \; \tilde{T_i} \geq \tilde{S_i} \mid \bar{I}]\Pr[\bar{I}]}{\Pr[I]} \\
&\geq 1 -  r \cdot e^{-\left(\sqrt{Nw^d/2}-\sqrt{2nw^d}\right)^2} - \Pr[\bar{I}]\\
&\geq 1 -  r \cdot e^{-\left(\sqrt{Nw^d/2}-\sqrt{2nw^d}\right)^2} - e^{-\Theta(n+N)}.
\end{align*}
The final inequality follows from standard Poisson tail bounds \cite[Proposition 1]{glynn1987upper}.  Plugging in $N = C \left(n + \left(\frac{\sqrt{d}}{\epsilon}\right)^d \left( \log \frac{1}{\delta}+d\log \frac{d}{\epsilon}\right)\right)$, for a sufficiently large constant $C$, gives us the lemma.
\end{proof}

\begin{lemma}[Improved Nonuniform Sampling Lemma]
\label{lem:improved-nonuniform}
Let $d\in \N$ and let $x^*_1, \ldots, x^*_n, x_1, \ldots, x_N$ be drawn independently from a distribution $\D$ on $\X = [0,1]^d$. If $N \geq C m n^{(d+1)/2} d^{d/2} \left(\log \frac{m}{\delta} + d \log(nd)\right)$ for a sufficiently large constant $C$, then with probability at least $1 - \delta$, there is a set $\mathcal{J}$ of cardinality at least $n - \sqrt{(n \log \frac{2}{\delta}) / 2}$ for which, for each $x^*_j$ with $j \in \mathcal{J}$, there are at least $m$ points $x_{i_{j,1}}, \ldots, x_{i_{j,m}}$ satisfying $\|x^*_j - x_{i_{j,m}}\|_2 \leq \frac{1}{\sqrt{n}}$, and all the $i_{1,1}, \ldots, i_{1,m}, \ldots, i_{|\mathcal{J}|,1}, \ldots, i_{|\mathcal{J}|,m} \in [N]$ are distinct.
\end{lemma}
\begin{proof}
This follows from a combination of the proofs of Lemma~\ref{lemma:sampling} and Lemma~\ref{lem:improved-sampling}.  As in Lemma~\ref{lemma:sampling}, we partition $[0,1]^d$ into heavy and light cells of diameter $1/\sqrt{n}$ and see that with probability $\delta / 2$ at most $\sqrt{(n \log \frac{2}{\delta}) / 2}$ $x_i^*$'s fall in light cells.  We then apply the argument of Lemma~\ref{lem:improved-sampling} to these heavy cells, using $\frac{1}{r\sqrt{n}}$  (for $r =  (nd)^{d/2}$) as their probability mass instead of $w^d$.  Thus, it suffices that 
$$N\geq C m n^{(d+1)/2} d^{d/2} \left(\log \frac{m}{\delta} + d \log(nd)\right) \geq C' m \left(n + r\sqrt{n} \log \frac{mr}{\delta}\right)$$ 
for a sufficiently large constant $C$.
\end{proof}

\section{Proofs of excess risk bounds}
\label{app:excess-risk}

\subsection{Proof of Theorem ~\ref{thm:bound-imperfect-sampling}}
\label{app:slow-rates}

\begin{Proof}[Proof of Theorem~\ref{thm:bound-imperfect-sampling}]

For each $i \in [n]$ and $j \in [m]$, draw $Y_{i,j}$ independently according to distribution $\D_{X_i^*}$. This ``clean'' sample is simply a theoretical device for the analysis.

We first set up some convenient notation. 
For each $i \in [n]$, define $\mathbf{Y}_i := (Y_{i,1}, \ldots, Y_{i,m})$ and $\mathbf{\tilde{Y}}_i := (\tilde{Y}_{i,1}, \ldots, \tilde{Y}_{i,m})$. 
Let $\Prob$ be a probability measure operator, defined according to $\Prob \loss_f = \E [ \loss_f(X,\mathbf{Y}) ]$; here, $X$ is drawn from $\D$, and, conditionally on $X$, $\mathbf{Y} = (Y_1, \ldots, Y_m)$ is drawn from $(\D_X)^m$. For a fixed $f$, $\Prob$ takes $\loss_f$ to its expected value on a new draw from the distribution $X$ and an $m$-tuple $\mathbf{Y}$ from $\D_X$.
We also define the empirical probability measure operators $\Pn$ and $\cPn$ via 
\begin{align*}
\Pn\loss_f = \frac{1}{n} \sum_{i=1}^n \loss_f(X^*_j, \mathbf{Y}_i) 
\qquad \text{and} \qquad
\cPn\loss_f = \frac{1}{n} \sum_{i=1}^n \loss_f(X^*_i, \mathbf{\tilde{Y}}_i) .  
\end{align*}

Now according to Lemma \ref{lemma:union-argument} below, we have for any positive $t_1,t_2,t_3$,
\begin{align*}
&\Pr \left( \Prob \loss_{\cf} > \Prob \loss_{f^*} + t_1 + 2 t_2 + t_3 \right) \\
&\leq \Pr \left( \sup_{f \in \F} (\Prob - \Pn) \loss_f > t_1 \right) 
          + \Pr \left( \sup_{f \in \F} |(\Pn - \cPn) \loss_f| > t_2 \right) 
          + \Pr \bigl( (\Pn - \Prob) \loss_{f^*} > t_3 \bigr) .
\end{align*}
From Lemma~\ref{lemma:rademacher}, the first probability is at most $\delta / 4$ when $t_1 = 2 L \rad_n(\F) + B \sqrt{\frac{\log(4/\delta)}{2 n}}$. 
From Hoeffding's inequality, the third probability is at most $\delta / 4$ when $t_3 = B \sqrt{\frac{\log(4/\delta)}{2 n}}$ (note that $f^*$ is fixed). 
The remainder of the proof controls the second probability. As we will see, we will be able to take $t_2 = O \left( B \sqrt{\frac{\log(1/\delta)}{n}} \right)$ when the probability is at most $\delta/2$.

First, under Lemma~\ref{lem:improved-nonuniform}, 
with probability at least $1 - \delta/4$, 
there is a subset $\mathcal{I}_G \subset [n]$ of cardinality at least 
$n_G := n - \sqrt{(n \log \frac{8}{\delta}) / 2}$ 
for which, for each $X^*_i$ with $i \in \mathcal{I}_G$, 
there are at least $m$ points $X_{k_{i,1}}, \ldots, X_{k_{i,m}}$ within distance $\varepsilon$ of $x^*_i$, 
and all the $k_{1,1}, \ldots, k_{1,m}, \ldots, k_{n,1}, \ldots, k_{n,m} \in [N]$ are distinct.

Next, we make the observation that the \emph{observed} sample can be obtained by the following corruption modifications to $(\mathbf{Y}_i)_{i \in [n]}$.
\begin{enumerate}
\item For $i \in [n] \setminus \mathcal{I}_G$, draw $\tilde{Y}_{i,j}$ from distribution $\D_{X_{i,j}}$.
\item For $i \in \mathcal{I}_G$, observe that Assumption \eqref{eqn:tv-lipschitz} implies that, without loss of generality, we can view each $Y_{i,j}$ as drawn in the following way. First, set $\tilde{Y}_{i,j}$ to $Y_{i,j}$. Next, draw a Bernoulli random variable $Z_{i,j}$ with success probability $\tau := K \varepsilon$, and if $Z_{i,j} = 1$, we corrupt $\tilde{Y}_{i,j}$ by setting it (again) to a new draw from some distribution $Q_{i,j}$ that can depend on both $X^*_i$ and $X_{i,j}$.
\end{enumerate}

For each $i$, if $Z_{i,j} = 0$, we say that $(i,j)$ are $\good$, and if $(i,1), \ldots, (i,m)$ all are $\good$, we say that $i$ is $\good$. If some $i$ is not $\good$, then it is $\bad$. 
Clearly, for each $i$ separately, with probability at least $1 - m K / \sqrt{n}$ over $(Z_{i,j})_{j \in [m]}$ it holds that $i$ is $\good$ (recall that $m \tau = m K \varepsilon = m K / \sqrt{n}$). 
Thus, from Hoeffding's inequality the probability (over $(Z_{i,j})_{i \in [n], j \in [m]}$) that at least $(C  +1) \frac{m K n_G}{\sqrt{n}}$ of the $i$'s are $\bad$ is at most $e^{-2 (n_G / n) (m K C)^2} \leq e^{-C^2}$ (recall that $n \geq 2 \log \frac{8}{\delta}$) , so if $C = \sqrt{\log \frac{4}{\delta}}$ then this probability is at most $\delta/4$ (and our total probability of failure thus far is $\delta/2$). We denote the (further diminished) good set of indices by $\mathcal{I}'_G := \left\{i \in \mathcal{I}_G \colon i \text{ is } \good \right\}$; this set has cardinality at least $n_G' := n_G \left( 1 - \frac{\sqrt{\log \frac{4}{\delta}} + m K}{\sqrt{n}} \right)$ with probability at least $1 - \delta/2$.

From the above argument, we see that with probability at least $1 - \delta/2$, 
at most $n'_B := n - n'_G$ corruption modifications occurred, and hence
\begin{align*}
\sup_{f \in \F} \left|
    \frac{1}{n} \sum_{i=1}^n \loss_f(X^*_i, \mathbf{\tilde{Y}}_i)
    - \frac{1}{n} \sum_{i=1}^n \loss_f(X^*_i, \mathbf{Y}_i) 
\right| 
\leq \frac{B (n - n'_G)}{n} .
\end{align*}

Observe that
\begin{align*}
n - n'_G 
&= n 
      - \left( n - \sqrt{(n \log \frac{8}{\delta}) / 2}  \right) 
         \left( 1 - \frac{\sqrt{\log \frac{4}{\delta}} + m K}{\sqrt{n}} \right) \\
&= \sqrt{n} \left( \sqrt{\log \frac{4}{\delta}} + m K 
                            + \sqrt{\frac{\log \frac{8}{\delta}}{2}} \right) 
       - \sqrt{\frac{\log \frac{8}{\delta}}{2}} \left( \sqrt{\log \frac{4}{\delta}} + m K \right) \\
&\leq \sqrt{n} \left( 2 \sqrt{\log \frac{4}{\delta}} + m K \right) ,
\end{align*}
and thus we may take $t_2 = \frac{B \left( 2 \sqrt{\log \frac{4}{\delta}} + m K \right)}{\sqrt{n}}$.
\end{Proof}

\begin{lemma} \label{lemma:union-argument}
  Under the hypotheses of Theorem \ref{thm:bound-imperfect-sampling} we have the following, with probability taken over the random sample (i.e.~$\cf$, $\Pn$, and $\cPn$ are functions of the random sample):
  \begin{align*}
  &\Pr \left( \Prob \loss_{\cf} > \Prob \loss_{f^*} + t_1 + 2 t_2 + t_3 \right) \\
  &\leq \Pr \left( \sup_{f \in \F} (\Prob - \Pn) \loss_f > t_1 \right) 
            + \Pr \left( \sup_{f \in \F} |(\Pn - \cPn) \loss_f| > t_2 \right) 
            + \Pr \bigl( (\Pn - \Prob) \loss_{f^*} > t_3 \bigr) .
  \end{align*}
\end{lemma}
\begin{proof}
First, observe that (using $\ind{E}$ for the 0-1 indicator function of a random event)
\begin{align*}
&\ind{ 
      \left( \sup_{f \in \F} (\Prob - \Pn) \loss_f \leq t_1 \right) 
      \opmin \left( \sup_{f \in \F} |(\Pn - \cPn) \loss_f| \leq t_2 \right) 
      \opmin \bigl( (\Pn - \Prob) \loss_{f^*} \leq t_3 \bigr) 
} \\
&\leq \ind{ \Prob \loss_{\cf} \leq \Prob \loss_{f^*} + t_1 + 2 t_2 + t_3 } .
\end{align*}
To see this,
\begin{align*}
\Prob \loss_{\cf}
\leq \Pn \loss_{\cf} + t_1 
&\leq \Pn \loss_{\hat{f}} + \Pn (\loss_{\cf} - \loss_{\hat{f}}) + t_1 \\
&\overset{(\text{a})}{\leq} \Pn \loss_{\hat{f}} + 2 t_2 + t_1 \\
&\overset{(\text{b})}{\leq} \Pn \loss_{f^*} + 2 t_2 + t_1 \\
&\leq \Prob \loss_{f^*} + t_3 + 2 t_2 + t_1 ,
\end{align*}
where (a) is from Lemma~\ref{lemma:loss-cf-to-loss-fhat} and (b) is from the optimality of ERM under $\Pn$.

By subtracting each side from one and rearranging, we get an implication on the negation of these events
\begin{align*}
&\ind{ 
      \left( \sup_{f \in \F} (\Prob - \Pn) \loss_f > t_1 \right) 
      \opmax \left( \sup_{f \in \F} |(\Pn - \cPn) \loss_f| > t_2 \right) 
      \opmax \bigl( (\Pn - \Prob) \loss_{f^*} > t_3 \bigr) 
} \\
&\geq \ind{ \Prob \loss_{\cf} > \Prob \loss_{f^*} + t_1 + 2 t_2 + t_3 }
\end{align*}
and we can use the union bound.
\end{proof}

\begin{lemma}\label{lemma:loss-cf-to-loss-fhat}
The following statement is true:
\begin{align*}
\Pn \loss_{\cf} - \Pn \loss_{\hat{f}} 
\leq 2 \sup_{f \in \F} \left| (\Pn - \cPn) \loss_f \right| .
\end{align*}
\end{lemma}
\begin{proof}
Observe that
\begin{align*}
\Pn \loss_{\cf} - \Pn \loss_{\hat{f}} 
&= \left( \Pn \loss_{\cf} - \cPn \loss_{\cf} \right) + \left( \cPn \loss_{\cf} - \Pn \loss_{\hat{f}} \right) \\
&\leq \left( \Pn \loss_{\cf} - \cPn \loss_{\cf} \right) + \left( \cPn \loss_{\hat{f}} - \Pn \loss_{\hat{f}} \right) \\
&\leq 2 \sup_{f \in \F} \left| (\Pn - \cPn) \loss_f \right| ,
\end{align*}
where the first inequality is from the optimality of $\cf$ under $\cPn$.
\end{proof}

\subsection{Proof of Theorem~\ref{thm:bound-faster}}
\label{app:fast-rates}

\begin{Proof}[Proof of Theorem~\ref{thm:bound-faster}]
For each $i \in [n]$ and $j \in [m]$, draw $Y_{i,j}$ independently according to distribution $\D_{X_i^*}$. This ``clean'' sample is simply a theoretical device for the analysis.

We first set up some convenient notation. 
For each $i \in [n]$, define $\mathbf{Y}_i := (Y_{i,1}, \ldots, Y_{i,m})$ and $\mathbf{\tilde{Y}}_i := (\tilde{Y}_{i,1}, \ldots, \tilde{Y}_{i,m})$. 
Let $\Prob$ be a probability measure operator, defined according to $\Prob \loss_f = \E [ \loss_f(X,\mathbf{Y}) ]$; here, $X$ is drawn from $\D$, and, conditionally on $X$, $\mathbf{Y} = (Y_1, \ldots, Y_m)$ is drawn from $(\D_X)^m$. For a fixed $f$, $\Prob$ takes $\loss_f$ to its expected value on a new draw from the distribution $X$ and an $m$-tuple $\mathbf{Y}$ from $\D_X$.
We also define the empirical probability measure operators $\Pn$ and $\cPn$ via 
\begin{align*}
\Pn\loss_f = \frac{1}{n} \sum_{i=1}^n \loss_f(X^*_j, \mathbf{Y}_i) 
\qquad \text{and} \qquad
\cPn\loss_f = \frac{1}{n} \sum_{i=1}^n \loss_f(X^*_i, \mathbf{\tilde{Y}}_i) .  
\end{align*}

Now according to Lemma \ref{lemma:union-argument-II} below, we have for any positive $t_1, t_2$,
\begin{align*}
&\Pr \left( \Prob \loss_{\cf} > \Prob \loss_{f^*} + t_1 + 2 t_2 \right) \\
&\leq \Pr \left( \sup_{f \in \F} \left\{ \Prob (\loss_f - \loss_{f^*}) - \Pn (\loss_f - \loss_{\hat{f}}) \right\} > t_1 \right) 
          + \Pr \left( \sup_{f \in \F} |(\Pn - \cPn) \loss_f| > t_2 \right) .
\end{align*}
Now, since the risk is $\sigma$-strongly convex, the first probability is at most $\delta / 3$ from Theorem 1 of \cite{sridharan2009fast} with $a = 1$ and $\lambda = 2 \sigma$, yielding the choice $t_1 = 8 L^2 B^2 \left( 32 + \log(3/\delta) \right) / (\sigma n)$.

The remainder of the proof controls the second probability. As we will see, we will be able to take $t_2 = O \left( \frac{B \log(1/\delta)}{n} \right)$ when the probability is at most $2 \delta/3$.

First recall that by using Algorithm~\ref{alg:improved} with our choice of parameters, Lemma~\ref{lem:improved-sampling} gives us that with probability at least $1 - \delta/3$, for all $i \in [n]$, 
there are at least $m$ points $X_{k_{i,1}}, \ldots, X_{k_{i,m}}$ within distance $\varepsilon$ of $X^*_i$, 
and all the $k_{1,1}, \ldots, k_{1,m}, \ldots, k_{n,1}, \ldots, k_{n,m} \in [N]$ are distinct.

Next, we make the observation that the \emph{observed} sample can be obtained by the following ``corruption'' modifications to $(\mathbf{Y}_i)_{i \in [n]}$.
\begin{enumerate}
\item For $i \in [n] \setminus \mathcal{I}_G$, draw $\tilde{Y}_{i,j}$ from distribution $\D_{X_{i,j}}$.
\item For $i \in \mathcal{I}_G$, observe that Assumption \eqref{eqn:tv-lipschitz} implies that, without loss of generality, we can view each $Y_{i,j}$ as drawn in the following way. First, set $\tilde{Y}_{i,j}$ to $Y_{i,j}$. Next, draw a Bernoulli random variable $Z_{i,j}$ with success probability $\tau := m K \varepsilon$, and if $Z_{i,j} = 1$, we ``corrupt'' $\tilde{Y}_{i,j}$ by setting it (again) to a new draw from some distribution $Q_{i,j}$ that can depend on both $X^*_i$ and $X_{i,j}$.
\end{enumerate}

For each $i$, if $Z_{i,j} = 1$, we say that $(i,j)$ are $\good$, and if $(i,1), \ldots, (i,m)$ all are $\good$, we say that $i$ is $\good$. If some $i$ is not $\good$, then it is $\bad$. 
Clearly, for each $i$ separately, with probability at least $1 - 1/n$ over $(Z_{i,j})_{j \in [m]}$ it holds that $i$ is $\good$ (recall that $m \tau = m K \varepsilon = 1/n$). 
Thus, from a multiplicative Chernoff bound\footnote{The bound being used is $\Pr(S_n \geq R) \leq 2^{-R}$ for $R \geq 6 \E [ S_n ]$, from equation (4.3) of \cite{mitzenmacher2005probability}, where $S_n$ is the sum of i.i.d.~Bernoulli random variables with success probability $1/n$.} the probability (over $(Z_{i,j})_{i \in [n], j \in [m]}$) that at least $\frac{3}{2} \log \frac{3}{\delta}$ of the $i$'s are $\bad$ is at most $\delta/3$ (for $ \delta < 3 e^{-4}$) (and our total probability of failure thus far is $2 \delta/3$). We denote the good set of indices by $\mathcal{I}'_G := \left\{i \in [n] \colon i \text{ is } \good \right\}$; this set has cardinality at least $n'_G = n - \frac{3}{2} \log \frac{3}{\delta}$ with probability at least $1 - 2 \delta/3$.

From the above argument, we see that with probability $1 - 2 \delta/3$, 
at most $n'_B := n - n'_G$ corruption modifications occur, and hence
\begin{align*}
\sup_{f \in \F} \left|
    \frac{1}{n} \sum_{i=1}^n \loss_f(X^*_i, \mathbf{\tilde{Y}}_i)
    - \frac{1}{n} \sum_{i=1}^n \loss_f(X^*_i, \mathbf{Y}_i) 
\right| 
\leq \frac{B (n - n'_G)}{n} .
\end{align*}

Thus, we may take $t_2 = \frac{\frac{3}{2} B \log \frac{3}{\delta}}{n}$.
\end{Proof}

\begin{lemma} \label{lemma:union-argument-II}
  Under the hypotheses of Theorem \ref{thm:bound-imperfect-sampling} we have the following, with probability taken over the random sample (i.e.~$\cf$, $\Pn$, and $\cPn$ are functions of the random sample):
  \begin{align*}
  &\Pr \left( \Prob \loss_{\cf} > \Prob \loss_{f^*} + t_1 + 2 t_2 \right) \\
  &\leq \Pr \left( 
\sup_{f \in \F} \left\{ \Prob (\loss_f - \loss_{f^*}) - \Pn (\loss_f - \loss_{\hat{f}}) \right\} > t_1 \right) 
            + \Pr \left( \sup_{f \in \F} |(\Pn - \cPn) \loss_f| > t_2 \right) .
  \end{align*}
\end{lemma}
\begin{proof}
First, observe that (using $\ind{E}$ for the 0-1 indicator function of a random event)
\begin{align*}
&\ind{ 
      \left( \sup_{f \in \F} \left\{ \Prob (\loss_f - \loss_{f^*}) - \Pn (\loss_f - \loss_{\hat{f}}) \right\} \leq t_1 \right) 
      \opmin \left( \sup_{f \in \F} |(\Pn - \cPn) \loss_f| \leq t_2 \right) 
} \\
&\leq \ind{ \Prob \loss_{\cf} \leq \Prob \loss_{f^*} + t_1 + 2 t_2 } .
\end{align*}
To see this,
\begin{align*}
\Prob \loss_{\cf}
\leq \Prob \loss_{f^*} + \Pn ( \loss_{\cf} - \loss_{\hat{f}} ) + t_1 
\leq \Prob \loss_{f^*} + 2 \sup_{f \in \F} \left| (\Pn - \cPn) \loss_f \right| + t_1 ,
\end{align*}
where the second inequality is from Lemma~\ref{lemma:loss-cf-to-loss-fhat}.

By subtracting each side from one and rearranging, we get an implication on the negation of these events
\begin{align*}
&\ind{ 
      \left( \sup_{f \in \F} \left\{ \Prob (\loss_f - \loss_{f^*}) - \Pn (\loss_f - \loss_{\hat{f}}) \right\} > t_1 \right) 
      \opmax \left( \sup_{f \in \F} |(\Pn - \cPn) \loss_f| > t_2 \right) 
} \\
&\geq \ind{ \Prob \loss_{\cf} > \Prob \loss_{f^*} + t_1 + 2 t_2 + t_3 }
\end{align*}
and we can use the union bound.
\end{proof}

\section{Lower Bounds for High Dimensions} \label{sec:lower}

\newcounter{temptheorem2}
\setcounter{temptheorem2}{\value{theorem}}
\setcounter{theorem}{\value{booldimlowercntr}}
\booldimlower*
\setcounter{theorem}{\value{temptheorem2}}
\begin{proof}
  We consider the simplest nontrivial hypothesis class, constant functions (i.e.~the set $\{f_c: c \in \reals\}$ where each $f_c: x \mapsto c ~~(\forall x)$.
  The instances we construct will be realizable, i.e.~in each instance, there will exist a constant $c$ such that $\Var(y \mid x) = c ~~(\forall x)$.

  Consider the discrete uniform distribution on the boolean hypercube $\X = \{0,1\}^d$.
  We have $\Y = \{0,1\}$.
  In instance $A$, $\Pr[y=1 \given x] = 0.5$ independently for all $x$, and $\Var(y \given x) = 0.25$ for all $x$.
  In the family of instances $\B$, we construct an instance $B$ by drawing, for each $x$, $\theta_x \in \{0,1\}$ uniformly and independently at random; then conditioned on $x$, the distribution of $y$ is given by $\Pr[y = 1 \given x] = \theta_x$.
  Notice that for all instances in the family $\B$, $\Var(y \given x) = 0$ for every $x$.
  Also, the Lipschitz constant satisfied by instances in $\B$ is $1$, as all distinct $x$ lie at a distance at least $1$ from each other.

  \emph{Informal sketch.}
  Any algorithm that is accurate must, with probability close to $1$, produce a different output when given access to $A$ than when given access to a uniform instance from $B$.
  However, by the principle of deferred decisions, we can rewrite the algorithm's behavior in the latter case as follows: Each time a uniformly random $x$ is drawn, if the algorithm has already seen a sample from $x$, then set $y$ consistent with that previous sample; otherwise, draw $\theta_x$ uniformly at random and set $y = \theta_x$.

  Thus, the input to the algorithm is distributed exactly identically in both cases unless the algorithm obtains multiple samples from the same $x$.
  However, with $o(2^{0.5d})$ samples, the probability of this occurring is $o(1)$ (by the ``birthday paradox''), so the algorithm has the same distribution of outputs with probability $1 - o(1)$.

  \emph{Formal proof.}
  Let $M$ be an algorithm and write $M(A)$ for the random variable which is $M$'s hypothesis when run on samples from $A$, while $M(B)$ is $M$'s hypothesis when run on samples from a uniformly randomly chosen instance $B$ from the family $\B$.
  Suppose that $M$ satisfies that, with probability at least $\frac{2}{3}$, its hypothesis (which is some constant $c$) is within $\epsilon$ of the correct variance, i.e.~$M(A) \geq 0.25-\epsilon$ and $M(B) \leq \epsilon$ each with probability at least $\frac{2}{3}$.\footnote{One can also state this condition as an $\approx \epsilon$ generalization error guarantee for $M$ with appropriate loss function.}
  Suppose $\epsilon < 0.125$ and the number of samples drawn by $M$ is $o(2^{0.5d})$; we show a contradiction.

  Use $s$ to denote a set of samples each of the form $(x,y)$, and let $NR$ denote those sample-sets which have ``no repeated $x$'s'', i.e.~$NR = \{s : \text{each $x$ in $s$ is unique}\}$.
  Use $\Pr_A[s]$ to denote the probability of drawing a set of samples $s$ given access to $A$, with $\Pr_B[s]$ the probability of drawing $s$ given access to a uniformly random instance from $B$, and so on.
  We have
  \begin{align}
    \frac{2}{3}
      &\leq \Pr[ M(A) \geq 0.25-\epsilon ] \nonumber \\
      &=    \sum_s \Pr_A[s] \Pr[ M(s) \geq 0.25-\epsilon] \nonumber \\
      &=    \sum_{s \in NR} \Pr_A[s] \Pr[ M(s) \geq 0.25-\epsilon ]  ~ + ~ \sum_{s \not\in NR} \Pr_A[s] \Pr[ M(s) \geq 0.15]  \nonumber \\
      &\leq \sum_{s \in NR} \Pr_A[s] \Pr[ M(s) \geq 0.25-\epsilon ]  ~ + ~ \sum_{s \not\in NR} \Pr_A[s] .  \label{eqn:birthday-lower-1}
  \end{align}
  Now, for all $s \in NR$, we claim $\Pr_A[s] = \Pr_B[s]$, as each is equal to
  \begin{align*}
    \prod_{(x,y) \in s} \Pr[x] \Pr[y \given x] &= \prod_{(x,y) \in s} 2^{-d} \left(\frac{1}{2}\right).
  \end{align*}
  (In the case of $A$, this is immediate; in the case of $B$, by the principle of deferred decisions, we can construct $B$ piece-by-piece; as each sample $(x,y) \in s \cap NR$ is drawn, we draw $\theta_x \in \{0,1\}$ uniformly at random and set $\Pr[y = 1] = \theta_x$, which results in a uniform distribution on $y$.)

  Meanwhile, $\sum_{s \not\in NR} \Pr_A[s]$ is the probability of drawing a sample with some repeated $x$ value, which we claim is $o(1)$ with $o(2^{0.5d})$ samples.
  The distribution is uniform on the $2^d$ possible $x$ values.
  The probability of a repeat or ``collision'', by Markov's inequality, is at most the expected number of collisions; with $m$ samples, there are ${m \choose 2}$ pairs each with a $2^{-d}$ chance of collision, so the expected number of collisions is $O\left(\frac{m^2}{2^d}\right)$, which is $o(1)$ for $m = o(2^{0.5d})$.

  So with $o(2^{0.5d})$ samples, we have by (\ref{eqn:birthday-lower-1}) that
  \begin{align*}
    \frac{2}{3}
      &\leq \sum_{s \in NR} \Pr_B[s] \Pr[ M(s) \geq 0.25-\epsilon ] ~ + ~ o(1) \\
      &\leq \Pr[ M(B) \geq 0.25-\epsilon ] + o(1)
  \end{align*}
  which, if $\epsilon < 0.125$, implies that $\Pr[ M(B) \leq \epsilon ] \leq \frac{1}{3} + o(1)$.
  This contradicts the accuracy assumption that $\Pr[ M(B) \leq \epsilon] \geq \frac{2}{3}$, so with this small number of samples, no such accurate $M$ exists.
\end{proof}
  
\begin{theorem} \label{thm:uniform-dimension-lower}
  With a uniform distribution on the unit hypercube $[0,1]^d$ in $d$ dimensions, with Lipschitz constant $K = d$, there is no algorithm for regression on variance with nontrivial accuracy drawing $o(2^{0.5d})$ samples.

  More precisely, estimating average variance over the hypercube to accuracy $\epsilon < \frac{1}{32}$ with success rate at least $\frac{2}{3}$ requires $\Omega(2^{0.5d})$ samples.
\end{theorem}
\begin{proof}
  The construction is very similar to the Boolean hypercube above.
  We have $\Y = \{0,1\}$.
  On instance $A$, for every $x$, $\Pr[ y = 1 \given x] = 0.5$.
  Hence, $\Var(y \given x) = 0.25$ for all $x$.

  \paragraph{Constructing $\B$.}
  We now construct the family of instances $\B$ and show that each has Lipschitz constant $K \leq d$ and has average variance $\E_x \Var(y \given x) \leq \frac{3}{16}$.

  In each instance, the hypercube is divided into ``corners'' and ``interior regions''.
  Let $\beta = \frac{1}{2} \frac{1}{2d}$ and let $\alpha = \frac{1}{2} - \beta$.
  Each ``corner'' $C_v$ is a hypercube of side length $\alpha$, inscribed in the unit hypercube and sharing the vertex $v \in \{0,1\}^d$.
In other words, $C_v = \{x : \|x - v\|_{\infty} \leq \alpha\}$.
  The portions of $[0,1]^d$ not contained in any corner are considered the ``interior regions''.

  To construct an instance in the family $\B$, draw $\theta_v \in \{0,1\}$ i.i.d. uniformly for each $v$.
  For points $x$ in some corner $C_v$, we have $\Pr[y = 1 \given x ] = \theta_v$.
  For points $x$ in the interior regions, let the notation $\|x - C_v\|_2$ denote $\min_{x' \in C_v} \|x' - x\|_2$.
  If there exists a $v$ such that $\|x-C_v\|_2 < \beta$, then, letting $r = \|x-C_v\|_2$, set $\Pr[y = 1 \given x] = \left(1 - \frac{r}{\beta}\right)\theta_v + \left(\frac{r}{\beta}\right)\left(\frac{1}{2}\right)$.
  (Note that this can only be true for at most one $v$, as this implies $\|x - v\|_{\infty} < \frac{1}{2}$, i.e.~$x$ is contained within the ``corner'' of side length $\frac{1}{2}$ touching $v$.)
  For all other $x$ (those not within $\beta$ of any $C_v$), we have $\Pr[ y = 1 \given x] = \frac{1}{2}$.

  Now we show that any instance in family $B$ has Lipschitz constant $K = d$.
  For shorthand, write $p_x := \Pr[y = 1 \given x]$.
  Note that total variation distance between the distributions on $y$ at $x$ and at $x'$ is $|p_x - p_{x'}|$.
  The Lipschitz constant is bounded by the maximal directional derivative of $p_x$ with respect to $x$ in any direction.
  This is zero if $x$ lies within some $C_v$ or if $x$ is not within distance $\beta$ of some $C_v$.
  Otherwise, if $x' = \argmin_{x'' \in C_v} \|x'' - x\|_2$, then the absolute value of the directional derivative is maximized in the direction $x' - x$, where it is $\frac{1}{2\beta}$.
  This gives a Lipschitz constant of $\frac{1}{2\beta}$.

  Now we bound the average variance of $y$ given $x$.
  The volume of each corner $C_v$ is $\alpha^d$ and there are $2^d$ corners, so the total volume of the corners is
  \begin{align*}
    2^d \alpha^d
      &= 2^d \left(\frac{1}{2} - \frac{1}{2}\frac{1}{d}\right)^d  \\
      &= \left(1 - \frac{1}{d}\right)^d  \\
      &\geq \left(1 - \frac{1}{2}\right)^2  \\
      &= \frac{1}{4}
  \end{align*}
  assuming $d \geq 2$.
  For any $x \in C_v$ for any $v$, $\Var(y \given x) = 0$.
  For all other $x$, $\Var(y \given x) \leq 0.25$; and the volume computation shows that they make up at most $\frac{3}{4}$ of the hypercube.
  Hence, average variance in this instance is at most $0.25 \left(\frac{3}{4}\right) = \frac{3}{16}$.

  (We note that, by letting $\beta = \frac{1}{2} \frac{1}{C d}$ for $C \geq 1$, the same derivation gives Lipschitz constant $K = \frac{C}{d}$ and an average variance bounded by $\frac{4C-1}{16C^2} \leq \frac{1}{4C}$.)

  \paragraph{Indistinguishability.}
  From here, the proof is almost identical to the Boolean case.
  Consider any algorithm $M$ that is with $\epsilon < \frac{1}{32}$ of the correct average variance with probability at least $\frac{2}{3}$.
  For each vertex $v$, let $R_v = \{x : \|x - v\|_{\infty} < \frac{1}{2}\}$.
  There are $2^d$ disjoint regions $R_v$, each with volume $\frac{1}{2^d}$, and with probability $1$, each sample $(x,y)$ has $x \in R_v$ for some $v$.

  Let $s$ denote a set of samples drawn by $M$ and let $NR = \{s: \text{$s$ has no repeated $x$s}\}$.
  If $M$ draws $o(2^d)$ samples, then $\Pr[s \in NR] = o(1)$.
  If $s \not\in NR$, then we claim $\Pr_A[s] = \Pr_B[s]$, where the notation is shorthand for the probability of drawing the samples $s$ given oracle access to $A$, or to a uniformly chosen member $B$ of $\B$, respectively.
  The reason is that, by the principle of deferred decisions, we can in the case of $B$ choose $\theta_v$ at the moment that a sample $(x,y)$ is drawn with $x \in R_v$, which occurs at most once for each $v$ because $s \not\in NR$.
  Because the distribution on $\theta_v$ is uniform $\{0,1\}$, for any $x \in R_v$, the distribution of $p_x = \Pr[y = 1 \given x]$ is uniform around $\frac{1}{2}$, or in other words, the unconditional probability that this $y = 1$ is exactly $\frac{1}{2}$.

  So,
  \begin{align*}
    \frac{2}{3}
      &\leq \Pr[ M(A) \geq 0.25 - \epsilon ]  \\
      &=    \sum_{s \in NR} \Pr_A[s] \Pr[M(s) \geq 0.25 - \epsilon] ~ + ~ o(1)  \\
      &=    \sum_{s \in NR} \Pr_B[s] \Pr[M(s) \geq 0.25 - \epsilon] ~ + ~ o(1)  \\
      &\leq \Pr[ M(B) \geq 0.25 - \epsilon ] ~ + ~ o(1) .
  \end{align*}
  For $\epsilon < \frac{1}{32}$, this contradicts the accuracy requirement that $\Pr[M(B) \leq \frac{3}{16} + \epsilon]$ with probability $\frac{2}{3}$.
\end{proof}

\section{Supporting arguments for Lipschitz regression lower bound}
\label{app:lipschitz-lower-bound}

In this section, we provide the full details linking the minimax lower bound of \cite{stone1982optimal} to our application in \S~\ref{sec:stone-lower-bound}. All notation in this section is from \cite{stone1982optimal}. 

Take $\X = [0, 1]^d$ and let $\D$ be the uniform distribution over $\X$. For all $x \in \X$, the distribution $\D_x$ is a certain subclass of Bernoulli distributions\footnote{See Condition 2 of \cite{stone1982optimal} for details; in our setting, we have $t = \theta(x) = \E [ Y \mid X = x ]$, so the Bernoulli distribution can vary with $x$, as further explained on p.~1350 of \citep{stone1980optimal}.} with $x \mapsto \E [ Y \mid X = x ]$ a $K$-Lipschitz function.
 Then for any estimator $\hat{g}$, there exists a law $Y \mid X$ satisfying the aforementioned assumptions such that
\begin{align*}
  \E \left[ ( \hat{g}(X) - \E [ Y \mid X ] )^2 \right] = \Omega \left( n^{-2 / (2 + d)} \right) ,
\end{align*}
and there is a matching upper bound, so that this is the minimax optimal rate of convergence for this problem. 

To see how the above result follows from \cite{stone1982optimal}, we take $T(\theta) = \theta$ and observe that the Lipschitz condition is reflected in Stone's equation (1.2) by setting $K_2$ to our $K$, $k$ to zero, and $\beta = 1$ (so that $p = k + \beta = 1$ and hence $r = 1 / (2 + d)$). Since $\theta$ is the parameter of a Bernoulli distribution, we need to verify from the lower bound construction that for all choices of $\theta \in \Theta_n$ used in the lower bound, we have $\theta \in (0, 1)$. We now do this verification.

As mentioned in the proof of Lemma 1 of \cite{stone1982optimal}, for any binary sequence $\tau_n \in \{0,1\}^{V_n}$, the corresponding\footnote{We use the notation $g^{(\tau)}_n$ rather than Stone's notation $g_n$ to make explicit the dependence on $\tau$.} $g^{(\tau)}_n$ vanishes at the boundary of $[0, 1]^d$. Moreoever, by assumption each $g_n^{(\tau)}$ is $K$-Lipschitz for $K = 1/(2 \sqrt{d})$. Let us verify that, for all $\theta \in \Theta_n$ and for all $x \in [0,1]^d$, it holds that $\theta(x) \in (0,1)$. Let $\theta_0 \equiv \frac{1}{2}$, and note that $\theta(x) = \theta_0(x) + g^{(\tau)}(x)$. It suffices to show that $|g^{(\tau)}_n(x)| \leq \frac{1}{2 \sqrt{2}}$ for all $\tau \in \{0,1\}^d$ and all $x \in [0,1]^d$. 
This is easily verified. First, observe that
\begin{align*}
|g^{(\tau)}_n(x)| 
= \left| g^{(\tau)}_n(x) - g^{(\tau)}_n(\mathbf{0}) \right| 
= \left| g^{(\tau)}_n(x) - g^{(\tau)}_n(\mathbf{1}) \right| ,
\end{align*}
since $g^{(\tau)}$ vanishing at the boundary of $[0,1]^d$ implies that $g^{(\tau)}_n(\mathbf{0}) = g^{(\tau)}_n(\mathbf{1}) = 0$. 
Therefore, 
\begin{align*}
|g^{(\tau)}_n(x)| 
&= \min \left\{ \left| g^{(\tau)}_n(x) - g^{(\tau)}_n(\mathbf{1}) \right|,
                         \left| g^{(\tau)}_n(x) - g^{(\tau)}_n(\mathbf{0}) \right| \right\} \\
&\leq K \min \left\{ \|x - \mathbf{1}\|,
                                 \|x - \mathbf{0}\| \right\} \\
&= K \left\|\frac{1}{2} \cdot \mathbf{1} \right\| \\
&= K \sqrt{d/2} \\
&= \frac{1}{2 \sqrt{2}} .
\end{align*}
The result then follows by applying Theorem 1 of \cite{stone1982optimal} with $q = 2$.

\section{Eliciting the Upper Confidence Bound}
\label{app:ucb}

Given a random variable $Y$, define $\ucb_\lambda(Y) = \E[Y] + \lambda\sigma[Y]$, where $\sigma[Y] = \sqrt{\Var[Y]} = \sqrt{\E[Y^2]-\E[Y]^2}$.
We show here the motivation behind the loss function used in \S~\ref{sec:experiments}.

The level sets of $\ucb_\lambda$, as a function of the \emph{law} of $Y$, i.e., the distribution, are given by
\begin{equation}
  \label{eq:elicit-ucb-1}
  r = \ucb_\lambda(Y) = \E[Y] + \lambda\sigma[Y]~.
\end{equation}

We can rewrite this as follows:
\begin{align}
  r - \E[Y] &= \lambda\sigma[Y]
  \\
  (r - \E[Y])^2 &= \lambda^2\Var[Y]\label{eq:elicit-ucb-6}
  \\
  0 &= \lambda^2\Var[Y] - \E[Y]^2 + 2r\E[Y] - r^2~,\label{eq:elicit-ucb-7}
\end{align}
though note that we have introduced another solution: both $r = \E[Y]+\lambda\sigma[Y]$ and $r=\E[Y]-\lambda\sigma[Y]$ now satisfy eq.~\eqref{eq:elicit-ucb-7} but only the former satisfies eq.~\eqref{eq:elicit-ucb-1}.
Apart from this spurious solution, the following would be an identification function for $\ucb_\lambda$, meaning a distribution has zero expectation if and only if it is in the level set for $r$; see~\citep{lambert2011elicitation,steinwart2014elicitation}.
\begin{equation}
  \label{eq:elicit-ucb-2}
  V(r,y_1,y_2) = \frac {\lambda^2} 2 (y_1 - y_2)^2 - y_1y_2 + (y_1+y_2)r - r^2~,
\end{equation}
whence
\begin{equation}
  \label{eq:elicit-ucb-3}
  \E[V(r,Y_1,Y_2)] = \lambda^2 \Var[Y] - \E[Y]^2 + 2\E[Y]r - r^2~,
\end{equation}
where of course $Y_1,Y_2\sim Y$ are independent.

Despite the fact that $V$ does not completely identify $\ucb_\lambda$, as for a given $r$, distributions where $\E[Y]-\sigma[Y] = r$ also satisfy $\E[V(r,Y_1,Y_2)]=0$, we can still try to integrate $-V$ with respect to $r$ to get a loss function.  The result is
\begin{equation}
  \label{eq:elicit-ucb-4}
  \loss(r,y_1,y_2) = -\left(\frac {\lambda^2} 2 (y_1 - y_2)^2 - y_1y_2\right)r - \frac 1 2 (y_1+y_2)r^2 + \frac 1 3 r^3~.
\end{equation}

As $V$ was not a true identification function, we know that $\E[\loss(r,Y_1,Y_2)]$ has multiple extrema, at $r=\E[Y]\pm \lambda\sigma[Y]$.
What's worse, since $\loss$ is cubic, we see that one can actually achieve arbitrarily negative loss as $r \to -\infty$.
Still, we can impose conditions on $r$ and $Y$ so that $r = \ucb_\lambda(Y)$ is the unique minimizer of $\E[\loss]$.
In particular, if we restrict $r$ to the range $\E[Y]-2\lambda\sigma[Y] \leq r \leq \infty$, then the loss will elicit $\ucb_\lambda$.
If we further restrict $r \geq \E[Y]-\lambda\sigma[Y]$, then $\E[\loss]$ is quasi-convex.
Finally, if we further restrict $r \geq \E[Y]$, then $\E[\loss]$ is convex.

Another possible loss is obtained by integrating $-rV$, that is, $r$ times $-V$.  This gives us,
\begin{align*}
  \loss(r,y_1,y_2) = -\frac 1 2 \left(\frac {\lambda^2} 2 (y_1 - y_2)^2 - y_1y_2\right) r^2 - \frac 1 3 (y_1+y_2) r^3 + \frac 1 4 r^4~.
\end{align*}
This removes the problem of unbounded negative loss for incorrect reports, but can still have a local optimum at $r = \E[Y] - \lambda\sigma[Y]$.

\section{Simulation Details}
\label{app:sims}
\paragraph{Algorithm implementation.} For Algorithms~\ref{alg:improved} and \ref{alg:nearby} we used $\epsilon = \frac{1}{2\sqrt{n}}$, which we found to generally perform the best.\footnote{For $\rho_k$ larger values of $\epsilon$ performed better for Algorithm~\ref{alg:nearby}, but caused it to be very computationally expensive.  This is because the number of metasamples is roughly $(n \epsilon)^m$ and for $\rho_4$, $m=5$.} There is some question about how to apply Algorithm~\ref{alg:improved} in the setting where one is given a fixed set of samples, rather than an oracle for drawing samples.  The strategy that we found to work best, which we used here, was given $\hat{n}$ samples, to use the first $n$ for $X_1^*,\ldots,X_n^*$ and the remaining $N = \hat{n} - n$ for $X^{(1)}_1,\ldots,X^{(1)}_N$.  We then binary searched to find the largest $n$ which allowed a perfect matching.\footnote{Note that this modification may introduce some bias so our theoretical results about Algorithm~\ref{alg:improved} no longer directly apply.  Nevertheless, we found this modification to be effective in practice.}  In all of our experiments, our two-observation methods used the linear functions as their hypothesis class.  In all cases, the true statistic is within this class. 

\paragraph{Empirical approach.} For $\ucb_\lambda$ we compared against the standard strategy used in practice.  This strategy is to sample $n / k$ random points $x_1,\ldots,x_{n/k}$ from $X$, and then for each $x_i$ sample $k$ values $y_{i,1},\ldots,y_{i,k} \sim Y | X\tinyeq x_i$.  For each $i$, we compute the empirical $\ucb_\lambda$ $u_i$ of $y_{i,1},\ldots,y_{i,k}$ and then fit a line to $(x_1,u_1),\ldots,(x_{n/k},u_{n/k})$ via least-squares regression.  We found that best results were achieved by letting $k = \sqrt{n}$.  We used an analogous strategy for $\rho_k$ where we empirically estimated $\rho_k$ using $k+1$ $y$ samples for each $x$.  

\paragraph{Evaluation framework.} As we used simple underlying statistics ($1$, $x+10$ and $x+2$) and simple hypothesis classes, we were able to compute the mean squared error between an algorithm's reported hypothesis and the true statistic via closed form expressions.  The only exception to this was for the two moment method of learning $\ucb_\lambda$, where the instead we estimated the mean squared error via 1000 sample Monte Carlo integration.  Each data point in Figure~\ref{fig:sims} is the median of 1000 independent trials.

\end{document}